\documentclass[journal]{IEEEtran}

\usepackage{cite}
\usepackage{amsmath,amssymb,amsfonts,amsthm}
\usepackage{algorithmic}
\usepackage{textcomp}

\usepackage{dsfont,eucal,bbm,bm,nicefrac} 
\usepackage{graphicx,float,subcaption,booktabs} 
	\graphicspath{{./figures/}}
\usepackage{algorithm,algorithmic} 
  \usepackage{tikz,xcolor} 

\newtheorem{theorem}{Theorem}
\newtheorem{corollary}{Corollary}[theorem]

\newtheorem{definition}{Definition}
\newtheorem{problem}{Problem}
\newtheorem{example}{Example}

\newcommand{\norm}[1]{\ensuremath{\left\| #1\right\|}}
\newcommand{\pbra}[1]{\ensuremath{\left( #1\right)}}
\newcommand{\sbra}[1]{\ensuremath{\left[ #1\right]}}
\newcommand{\cbra}[1]{\ensuremath{\left\{ #1\right\}}}
\newcommand{\abra}[1]{\ensuremath{\left< #1\right>}}

\newcommand{\pder}[2]{\ensuremath{\frac{\partial #1}{\partial #2}}}
\newcommand{\E}[1]{\ensuremath{\mathbb{E}\left[ #1\right]}}

\newcommand{\Eh}[1]{\ensuremath{\mathbb{\hat E}\left[ #1\right]}}

\DeclareMathOperator*{\argmax}{arg\,max}
\DeclareMathOperator*{\argmin}{arg\,min}

\begin{document}

\title{Online Deterministic Annealing  for Classification and Clustering}

\author{Christos N. Mavridis, \IEEEmembership{Member, IEEE}, and 
John S. Baras, \IEEEmembership{Life Fellow, IEEE}
\thanks{%
Manuscript published in the IEEE Transactions on Neural Networks and Learning Systems (TNNLS).} 
%
%
\thanks{
Research partially supported by the 
Defense Advanced Research Projects
Agency (DARPA) under Agreement No. HR00111990027, 
by the Office of Naval Research (ONR) grant N00014-17-1-2622, 
and by a grant from Northrop Grumman Corporation.%
}%
\thanks{The authors are with the 
Department of Electrical and Computer Engineering and 
the Institute for Systems Research, 
University of Maryland, College Park, USA.
{\tt\small emails:\{mavridis, baras\}@umd.edu}.}%
}

\maketitle
 \thispagestyle{empty}
\pagestyle{empty}

\begin{abstract}
Inherent in virtually every
iterative machine learning algorithm is the problem 
of hyper-parameter tuning which includes three major design parameters: 
(a) the complexity of the model, e.g., the number of neurons in a neural network, 
(b) the initial conditions, which heavily affect the 
behavior of the algorithm, and
(c) the dissimilarity measure used to quantify its performance.
%
We introduce an online prototype-based learning algorithm 
that can be viewed as a
progressively growing competitive-learning neural network architecture 
for classification and clustering.
The learning rule of the proposed approach 
is formulated as an online gradient-free stochastic approximation algorithm
that solves a sequence of appropriately defined optimization problems,
simulating an annealing process.
The annealing nature of the algorithm contributes to
avoiding poor local minima, 
offers robustness with respect to the initial conditions,
and provides a means 
to progressively increase the complexity of the learning model, 
through an intuitive bifurcation phenomenon.
The proposed approach 
is interpretable, requires minimal 
hyper-parameter tuning, and 
allows online control over the performance-complexity trade-off.
Finally, we show that Bregman divergences 
appear naturally as a family of dissimilarity measures 
that play a central role in both 
the performance and the computational complexity
of the learning algorithm.
%
%
\end{abstract}

\begin{IEEEkeywords}
Machine learning algorithms, progressive learning, annealing optimization, 
classification, clustering, Bregman divergences.
\end{IEEEkeywords}

\section{Introduction}
\label{Sec:Introduction}

\IEEEPARstart{L}{earning} 
from data samples has become an important component of 
artificial intelligence. 
While virtually all learning problems can be formulated as 
constrained stochastic optimization problems, 
the optimization methods can be intractable, typically 
dealing with mixed constraints and very large, or even infinite-dimensional spaces \cite{bennett2006interplay}.
For this reason, feature extraction,  
model selection and design, and analysis of optimization methods, 
have been the cornerstone of machine learning algorithms 
from their genesis until today.

Deep learning methods,
currently dominating the field of machine learning
due to their performance in multiple applications, 
attempt to learn feature representations from data, 
using biologically-inspired models in artificial neural networks 
\cite{krizhevsky2012imagenet,lee2009convolutional}. 
However, they typically use overly complex models of a great many parameters,
which comes in the expense of time, energy, data, memory, 
and computational resources \cite{thompson2020computational,strubell2019energy}. 
Moreover, they are, by design, hard to interpret and vulnerable to small perturbations and adversarial attacks
\cite{szegedy2013intriguing,carlini2017towards}.   
The latter, has led to an emerging hesitation in 
their implementation outside common benchmark datasets 
\cite{sehwag2019analyzing}, and, especially, in security critical applications.
On the other hand, it is understood that the trade-off between 
model complexity and performance is closely related to 
over-fitting, generalization, and robustness to 
input noise and attacks 
\cite{xu2012robustness}.
In this work, we introduce a learning model 
that progressively adjusts its complexity, 
offering online control over this trade-off.
The need for such approaches is reinforced by recent studies 
revealing that existing flaws in the current 
benchmark datasets may have inflated the need for overly
complex models \cite{northcutt2021pervasive}, and that
over-fitting to adversarial training examples 
may actually hurt generalization \cite{raghunathan2019adversarial}. 

We focus on prototype-based models,  
mainly represented by vector quantization methods, 
\cite{Kohonen1995,devroye2013,mavridis2020convergence}. 
In vector quantization, 
originally introduced as a signal processing method for compression, 
a set of codevectors (or prototypes) $M:=\cbra{\mu_i}$, 
is used to represent the data space in an optimal 
way according to an average distortion measure:
%
\begin{align*}
\min_{M} ~ J(M) := \E{\min_i d(X,\mu_i)},
\end{align*}
%
where the proximity measure $d$ defines the similarity between the 
random input $X$ and a codevector $\mu_i$.
The codevectors can be viewed as a set of neurons,
the weights of which live in the data space itself, 
and constitute the model parameters. 
In this regard, vector quantization algorithms can be viewed as
competitive-learning neural network architectures
with a number of appealing properties:
%
they are consistent, 
data-driven, interpretable, robust, topology-preserving
\cite{martin_topologyPreservationInSOM_2008}, 
sparse in the sense of memory complexity, and fast to train and evaluate.
In addition, 
they have recently 
shown impressive robustness 
against adversarial attacks,
suggesting suitability in security critical 
applications \cite{saralajew2019robustness}, while
their representation of the input 
in terms of memorized exemplars is an intuitive approach 
which parallels similar concepts from cognitive psychology
and neuroscience. 
%
As iterative learning algorithms, however, 
their behavior heavily depends on three major design parameters: 
(a) the number of neurons/prototypes,
which, defines the complexity of the model, 
(b) the initial conditions, 
that affect the transient and steady-state behavior of the algorithm, and
(c) the proximity measure $d$ used to quantify the similarity between
two vectors in the data space. 

Inspired by the deterministic annealing approach
\cite{rose1998deterministic}, 
we propose a learning approach that resembles an annealing process,
tending to avoid poor local minima, 
offering robustness with respect to the initial conditions,
and providing a means to progressively increase the 
complexity of the learning model, 
allowing online control over the performance-complexity trade-off.
We relax the original problem to a soft-clustering problem, 
introducing the association probabilities $p(\mu_i|X)$, and
replacing the cost function $J$ by 
$D(M) := \E{\sum_i p(\mu_i|X) d(X,\mu_i)}$.
This probabilistic framework 
(to be formally defined in Section \ref{Sec:ODA})
allows us to define 
the Shannon entropy $H(M)$ 
that 
characterizes the ``purity'' of the clusters induced by the codevectors.
We then replace the original problem 
by a \emph{sequence of optimization problems}:
%
\begin{align*}
\min_{M} ~ F_T(M) := D(M) - T H(M), 
\end{align*}
%
parameterized by a temperature coefficient $T$, which acts as a
Lagrange multiplier controlling the trade-off between minimizing 
the distortion $D$ and maximizing the entropy $H$.
By successively solving the optimization problems $\min_{M} ~ F_T(M)$ 
for decreasing values of $T$, the model undergoes a series of 
phase transitions that resemble an annealing process.
Because of the nature of the entropy term, 
in high temperatures $T$, the effect of the initial conditions is greatly
mitigated, while, as $T$ decreases, 
the optimal codevectors of the last optimization problem 
are used as initial conditions to the next,
which helps in avoiding poor local minima.
Furthermore, as $T$ decreases, 
the cardinality of the set of codevectors $M$ increases,
according to an intuitive bifurcation phenomenon.

Adopting the above optimization framework, we introduce an 
online training rule based on \emph{stochastic approximation} 
\cite{borkar2009stochastic}.
While stochastic approximation 
offers an online, adaptive, 
and computationally inexpensive optimization algorithm,
it is also strongly connected to dynamical systems.
This enables the study 
of the convergence of the learning algorithm through 
mathematical tools from dynamical systems and control 
\cite{borkar2009stochastic}.
We take advantage of this property to prove the convergence of the 
proposed learning algorithm as a 
consistent density estimator (unsupervised learning),
and a Bayes risk consistent classification rule (supervised learning).
Finally, we show that the proposed 
stochastic approximation learning algorithm
introduces \emph{inherent regularization mechanisms} and 
is also \emph{gradient-free}, provided that the proximity measure 
$d$ belongs to the family of Bregman divergences.
Bregman divergences are information-theoretic dissimilarity measures
that have been shown to play an important role in learning applications
\cite{banerjee2005clustering,villmann_onlineDLVQmath_2010},
including measures such as the widely used Euclidean distance and the 
Kullback-Leibler divergence.
We believe that these results can potentially lead to new developments
in learning with progressively growing models, including,
but not limited to, communication, control, and reinforcement learning applications \cite{mavridis2021progressive,mavridis2020detection,mavridis2021maximum}.

\section{Prototype-based Learning}
\label{Sec:PBLearning}

In this section, the mathematics and notation 
of prototype-based machine learning algorithms, 
which will be used as a base for our analysis,
are briefly introduced. 
For more details see 
\cite{banerjee2005clustering,biehl2016prototype,
villmann_onlineDLVQmath_2010,mavridis2020convergence}.

\subsection{Vector Quantization for Clustering}
\label{ssSec:VQ}

Unsupervised analysis
can provide valuable insights into the nature of the
dataset at hand, and it plays an important role in the
context of visualization.
Central to unsupervised learning
is the representation of 
data in a vector space by typical representatives, 
which is formally defined in the following optimization problem:

\begin{problem}
\label{prb:VQ}
Let $X: \Omega \rightarrow S\subseteq\mathbb{R}^d$ 
be a random variable defined in a 
probability space $\pbra{\Omega, \mathcal{F}, \mathbb{P}}$, 
and $d:S \times ri(S) \rightarrow \left[0,\infty\right)$
be a divergence measure, where $ri(S)$ represents the relative interior of $S$. 
Let $V := \cbra{S_h}_{h=1}^K$ be a partition of $S$ 
and $M := \cbra{\mu_h}_{h = 1}^K$ a set of codevectors, such that 
$\mu_h \in ri(S_h)$, for all $h=1, \ldots, K$.
A quantizer $Q:S \rightarrow M$ is defined as the random variable
$Q(X) = \sum_{h=1}^K \mu_h \mathds{1}_{\sbra{X \in S_h}}$
and the vector quantization problem is formulated as
\begin{align*}
	\min_{M,V} ~ J(Q) := \E{d\pbra{X,Q}}. 
\end{align*}	
\end{problem}

\noindent
Vector quantization is a hard-clustering algorithm,
and, as such, assumes that the quantizer $Q$ 
assigns an input vector $X$ to a unique codevector $\mu_h\in M$
with probability one.
As a result, Problem \ref{prb:VQ} becomes equivalent to
\begin{align}
\min_{\cbra{\mu_h}_{h = 1}^K} 
	\sum_{h=1}^K \E{d\pbra{X,\mu_h} 
	\mathds{1}_{\sbra{X \in S_h}}} 
\label{eq:HardMu}
\end{align}	
for V being a Voronoi partition, i.e., for 
%
\begin{equation*}
S_h = \cbra{x \in S: h = \argmin\limits_{\tau = 
		1,\ldots,K} ~ d(x,\mu_{\tau})},\ h =1,\ldots,K.
\label{eq:HardS}
\end{equation*}
%
%
It is typically the case that
the actual distribution of $X\in S$ is unknown, and 
a set of independent realizations 
$\cbra{X_i}_{i=1}^n:= \cbra{X(\omega_i)}_{i=1}^n$, 
for $\omega_i\in \Omega$, are available. 
In case the observations $\cbra{X_i}_{i=1}^n$ are available a priori,
the solution of the VQ problem is traditionally approached 
with variants of the LBG algorithm \cite{linde1980algorithm}, 
a generalization of 
the Lloyd algorithm \cite{sabin1986global} which includes 
the widely used $k$-means algorithm \cite{bottou1995convergence}. 
%

When the training data are not 
available a priori but are being observed online, 
or when the processing of the entire dataset 
in every optimization iteration is computationally infeasible, 
a stochastic vector quantization algorithm can 
be defined as a recursive asynchronous stochastic approximation algorithm
based on gradient descent \cite{mavridis2020convergence}:

\begin{definition}[Stochastic Vector Quantization (sVQ) Algorithm]
Repeat:
{\small
\begin{equation*}
\begin{cases}
	\mu_h^{t+1} &= \mu_h^t - \alpha(v(h,t)) 
				\mathds{1}_{\sbra{X_{t+1} \in S_h^{t+1}}} \nabla_{\mu_h^{}}{d\pbra{X_{t+1}, \mu_h^t}}\\
	S_h^{t+1} &= \cbra{X\in S: 
				h = \argmin\limits_{\tau = 1,\ldots,k} ~ d(X,\mu_{\tau}^t)},~h \in K
\end{cases}
\end{equation*}
}%
for $t\geq0$ until convergence, 
where $\mu_h^0$ is given during initialization, 
and $v(h,t)$ represents the number of times the component $\mu_h$ has
been updated up until time $t$.
\label{def:sVQ}
\end{definition}

\subsection{Learning Vector Quantization for Classification}
\label{sSec:LVQ}

The supervised counterpart of vector quantization is 
the particularly attractive and intuitive approach 
of the competitive-learning Learning Vector Quantization (LVQ) algorithm,
initially proposed by Kohonen \cite{Kohonen1995}.
LVQ for binary classification is formulated
in the following optimization problem (and generalized to any type of classification task, see, e.g. \cite{duda2012}):

\begin{problem}
\label{prb:lvq}
	Let the pair of random variables 
	$\cbra{X,c} \in S\times \cbra{0,1}$ defined in a probability space
	$\pbra{\Omega, \mathcal{F}, \mathbb{P}}$, with
	$c$ representing the class of $X$ and $S\subseteq\mathbb{R}^d$.	
	Let 
	$M := \cbra{\mu_h}_{h = 1}^K$, where $\mu_h \in ri(S_h)$
	represent codevectors, and 
	define the set $C_\mu := \cbra{c_{\mu_h}}_{h = 1}^K$,
	such that $c_{\mu_h} \in \cbra{0,1}$ 
	represents the class of $\mu_h$ for all $h \in \cbra{1,\ldots,K}$.
	The quantizer $Q^c:S \rightarrow \cbra{0,1}$ is defined such that 
	$Q^c(X) = \sum_{h=1}^k c_{\mu_h} \mathds{1}_{\sbra{X \in S_h}}$.	
	Then, the minimum-error classification problem is formulated as
	{\small	
	\begin{align*}
		\min_{\cbra{\mu_h,S_h}_{h = 1}^K} ~ J_B(Q^c) := 
		\pi_1 \sum_{H_0} \mathbb{P}_1\sbra{X\in S_h} +
			 \pi_0 \sum_{H_1} \mathbb{P}_0\sbra{X\in S_h} 
		%
		%
	\end{align*}	
	}%
	where $\pi_i := \mathbb{P}\sbra{c = i}, 
	\mathbb{P}_i\cbra{\cdot} := \mathbb{P}\cbra{\cdot | c = i}$,
	and $H_i$ is defined as $H_i := \cbra{h\in\cbra{1,\ldots,K}:Q^c=i}$, 
	$i\in\cbra{0,1}$.
\end{problem}

\noindent
LVQ algorithms that solve Problem \ref{prb:lvq}
are similar in structure with the stochastic 
vector quantization algorithm of Def. \ref{def:sVQ},
and make use of a modified distortion measure, 
which in the case of the original LVQ1 algorithm \cite{Kohonen1995}
takes the form:
\begin{align*}
d^l(x,c_x,\mu,c_\mu) = \begin{cases}
				d(x,\mu),~ c_x=c_\mu \\
				-d(x,\mu),~ c_x\neq c_\mu			
				\end{cases}
\end{align*}
Generalizations of this definition 
based on similar principles have also been proposed
\cite{sato1996generalized,hammer2002generalized}.

\subsection{Bregman Divergences as Dissimilarity Measures}

Prototype-based algorithms 
rely on measuring the proximity between different vector representations.
In most cases the Euclidean distance or another convex metric is used,
but this can be generalized to alternative dissimilarity measures inspired by 
information theory and statistical analysis, such as the Bregman 
divergences:
\begin{definition}[Bregman Divergence]
	Let $ \phi: H \rightarrow \mathbb{R}$, 
	be a strictly convex function defined on 
	a vector space $H$ such that $\phi$  
	is twice F-differentiable on $H$. 
	The Bregman divergence 
	$d_{\phi}:H \times H \rightarrow \left[0,\infty\right)$
	is defined as:
	\begin{align*}
		d_{\phi} \pbra{x, \mu} = \phi \pbra{x} - \phi \pbra{\mu} 
							- \pder{\phi}{\mu} \pbra{\mu} \pbra{x-\mu},
	\end{align*}
	where $x,\mu\in H$, and the continuous linear map 
	$\pder{\phi}{\mu} \pbra{\mu}: H \rightarrow \mathbb{R}$ 
	is the Fr\'echet derivative of $\phi$ at $\mu$.
	\label{def:BregmanD}
\end{definition}

Notice that, as a divergence measure, Bregman divergence can be used 
to measure the dissimilarity of one probability distribution to another on a statistical manifold, and is a notion weaker than that of the distance. 
In particular, it does not need to be symmetric or satisfy the triangle inequality.
In this work, we will concentrate on nonempty, compact convex sets 
$S\subseteq \mathbb{R}^d$ 
so that the derivative of $d_\phi$ with respect to the second argument 
can be written as
{\small
\begin{align*}
\pder{d_{\phi}}{\mu}(x,\mu) 
&= \pder{\phi(x)}{\mu} - \pder{\phi(\mu)}{\mu} 
- \pder{^2 \phi(\mu)}{\mu^2}(x-\mu) + \pder{\phi(\mu)}{\mu} \\
&= - \pder{^2 \phi(\mu)}{\mu^2}(x-\mu) 
= - \abra{\nabla^2 \phi(\mu),(x-\mu)}	
\end{align*}
}
where $x,\mu\in S$, $\pder{}{\mu}$ represents differentiation
with respect to the second argument of $d_{\phi}$, and 
$\nabla^2 \phi(\mu)$ represents the Hessian matrix of $\phi$ at $\mu$.
%

\begin{example}
As a first example, $\phi(x) = \abra{x,x},\ x\in\mathbb{R}^d$,
gives the squared Euclidean distance 
$$d_\phi(x,\mu) = \|x-\mu\|^2$$ 
for which 
$\pder{d_{\phi}}{\mu}(x,\mu) = -2(x-\mu)$.
\end{example}

\begin{example}
\label{ex:Idiv}
A second interesting Bregman divergence that shows the connection 
to information theory, is the generalized I-divergence 
which results from 
$\phi(x) = \abra{x,\log x},\ x\in\mathbb{R}_{++}^d$
such that  
$$d_\phi(x,y) = \abra{x,\log x - \log \mu}
	- \abra{\mathds{1}, x - \mu}$$
for which $\pder{d_{\phi}}{\mu}(x,\mu) = - diag^{-1}(\mu) (x-\mu)$,
where $\mathds{1}\in\mathbb{R}^d$ is the vector of ones, and $diag^{-1}(\mu)\in\mathbb{R}_{++}^{d\times d}$
is the diagonal matrix with diagonal elements the inverse elements of $\mu$.
It is easy to see that $\phi(x)$ reduces to the Kullback-Leibler divergence if 
$\abra{\mathds{1}, x} =1$.
\end{example}
%

The family of Bregman divergences provides proximity measures
that have been shown to enhance the performance of a learning algorithm
\cite{babiker2017using}.
%
%
In addition, the following theorem shows that the use of Bregman divergences 
is both necessary and sufficient 
for the optimizer $\mu_h$ of (\ref{eq:HardMu}) to be 
analytically computed as the expected value of the data inside $S_h$,
which is implicitly used by many ``centroid'' algorithms, 
such as $k$-means \cite{bottou1995convergence}:
\begin{theorem}
\label{thm:bregman}
	Let $X: \Omega \rightarrow S$ be a random variable defined in the 
	probability space $\pbra{\Omega, \mathcal{F}, \mathbb{P}}$ such that
	$\E{X} \in ri(S)$, and let a distortion measure 
	$d:S \times ri(S) \rightarrow \left[0,\infty\right)$, where 
	$ri(S)$ denotes the relative interior of $S$. 
	Then $\mu := \E{X}$
	is the unique minimizer of $ \E{d\pbra{X,s}} $ in $ri(S)$,	
	if and only if $d$ is a Bregman divergence for any function $\phi$
	that satisfies the definition.
\end{theorem}
\begin{proof}
For necessity, identical arguments as in Appendix B of 
\cite{banerjee2005clustering} are followed.
For sufficiency, 
{\small 
\begin{align*}
&\E{d_{\phi}(X,s)} - \E{d_{\phi}(X,\mu)} =\\
&=\phi(\mu) + \pder{\phi}{\mu}(\mu) \pbra{\E{X}-\mu} - \phi(s) - \pder{\phi}{s}(s) \pbra{\E{X}-s} \\
&=\phi(\mu) - \phi(s) - \pder{\phi}{s}(s)\pbra{\mu-s} 
=d_{\phi}\pbra{\mu,s} 
\geq 0, \quad \forall s \in S
\end{align*}%
}%
with equality holding only when $s=\mu$ by the strict convexity of $\phi$, which completes the proof.	
\end{proof}

\noindent
In Section \ref{Sec:ODA}, 
we will show a similar result for the proposed algorithm that uses 
a soft-partition approach.

\section{Online Deterministic Annealing for Unsupervised and Supervised Learning}
\label{Sec:ODA}

Online 
vector quantization algorithms, 
are proven to converge to locally optimal configurations
\cite{mavridis2020convergence}.
However, as iterative machine learning algorithms,
their convergence properties and final configuration 
depend heavily on two design parameters:
the number of neurons/clusters $K$, and 
their initial configuration.
Inspired by the deterministic annealing framework 
\cite{rose1998deterministic}, 
we relax the the original optimization problem (\ref{eq:HardMu})
to a soft-clustering problem, and replace it by 
a sequence of deterministic optimization problems, 
parameterized by a temperature coefficient,
that are progressively solved 
at successively reducing temperature levels.
As will be shown, the annealing nature of this algorithm 
will contribute to avoiding poor local minima, 
provide robustness with respect to the initial conditions,
and induce a progressive increase in the cardinality of the  
set of clusters needed to be used, via a intuitive bifurcation phenomenon.

\subsection{Soft-Clustering and Annealing Optimization}


In the clustering problem (Problem \ref{prb:VQ}), 
the distortion function $J$ is typically
non convex and riddled with poor local minima.
To partially deal with this phenomenon, soft-clustering approaches have been 
proposed as a probabilistic framework for clustering.
In this case, an input vector $X$ is assigned, through the quantizer 
$Q$, to all codevectors $\mu_h\in M$ with probabilities $p(\mu_h|X)$
, where $\sum_{h=1}^K p(\mu_h|X)=1$.
In this regard, the quantizer $Q:S\rightarrow M$ 
becomes a discrete random variable, 
with the set $M$ being its image, and can be fully described by the values 
of $M=\cbra{\mu_h}_{h=1}^K$ and the probability functions 
$\cbra{p(\mu_h|x)}_{h=1}^K$.
In contrast, hard clustering assumes that
$Q$ is a simple random variable that can be described fully by 
$M$ and $V=\cbra{S_h}_{h=1}^K$, since $p(\mu_h|X)=\mathds{1}_{\sbra{X \in S_h}}$
(see Problem \ref{prb:VQ}).

For the randomized partition we can rewrite the expected distortion as
\begin{align*}
D &= \E{d_\phi(X,Q)} \\
  &= \E{\E{d_\phi(X,Q)|X}} \\
  &= \int p(x) \sum_\mu p(\mu|x) d_\phi(x,\mu) ~dx
\end{align*}
where $p(\mu|x)$ is the association probability relating the input vector $x$
with the codevector $\mu$.
We note that, at the limit, where
each input vector
is assigned to a unique codevector with probability one,
this reduces to the hard clustering
distortion.
%
%
The main idea in deterministic annealing, 
is to seek the distribution that minimizes $D$ subject to a 
specified level of randomness, measured by the Shannon entropy
\begin{align*}
H(X,M) &= \E{-\log p(X,Q)} \\
       &= H(X) + H(Q|X) \\
       &= H(X) - \int p(x) \sum_\mu p(\mu|x) \log p(\mu|x) ~dx
\end{align*}
by appealing to Jaynes' maximum entropy principle%
\footnote{Informally, Jaynes' principle states: of all the probability distributions that satisfy a given set of constraints, 
choose the one that maximizes the entropy.}
\cite{jaynes1957information}.
This multi-objective optimization is conveniently formulated as the minimization of the Lagrangian
\begin{equation}
F = D-TH
\label{eq:F}
\end{equation}
where $T$ is the temperature parameter that acts as a Lagrange multiplier.
Clearly, for large values of $T$ we maximize the entropy, and, as $T$ is lowered,
we trade entropy for reduction in distortion.
Equation (\ref{eq:F}) also represents the scalarization method for trade-off 
analysis between two performance metrics \cite{miettinen2012nonlinear}. 
As $T$ varies we essentially transition from one Pareto point to another, 
and the sequence of the 
solutions will correspond to 
a Pareto curve of the multi-objective 
optimization (\ref{eq:F}) 
that resembles annealing processes in chemical engineering.
In this regard, the entropy $H$, which is closely related 
to the ``purity'' of the clusters,
acts as a regularization term which is given progressively less weight 
as $T$ decreases.
%

As in the case of vector quantization, 
we form a coordinate block optimization algorithm to minimize $F$,
by successively minimizing it with respect to the 
association probabilities $p(\mu|x)$ 
and the codevector locations $\mu$.
Minimizing $F$ with respect to the association probabilities $p(\mu|x)$ is 
straightforward and yields the Gibbs distribution 
\begin{equation}
p(\mu|x) = \frac{e^{-\frac{d(x,\mu)}{T}}}
			{\sum_\mu e^{-\frac{d(x,\mu)}{T}}},~ \forall x\in S
\label{eq:E}
\end{equation}
while, in order to minimize $F$ with respect to the codevector locations $\mu$ 
we set the gradients to zero 
\begin{equation}
\begin{aligned}
\frac d {d\mu} D = 0 
& \implies 
\frac d {d\mu} \E{\E{d(X,\mu)|X}} = 0 \\
& \implies
\int p(x) p(\mu|x) \frac d {d\mu} d(x,\mu) ~dx = 0
\end{aligned}
\label{eq:M}
\end{equation}

In the following theorem, we show that
we can have analytical solution to the last optimization step 
(\ref{eq:M}) in a convenient centroid form, if $d$ is a Bregman divergence.
This is a similar result to Theorem \ref{thm:bregman} for
vector quantization.

\begin{theorem}
Assuming the conditional probabilities $p(\mu|x)$ are fixed, 
the Langragian $F$ in (\ref{eq:F}) is minimized with respect 
to the codevector locations $\mu$ by  
\begin{equation}
\mu^* = \E{X|\mu} = \frac{\int x p(x) p(\mu|x) ~dx}{p(\mu)}
\label{eq:mu_star}
\end{equation}
if $d:=d_\phi$ is a Bregman divergence for some function 
$\phi$ that satisfies Definition \ref{def:BregmanD}.
\label{thm:bregman_in_DA}
\end{theorem}

\begin{proof}
If $d:=d_\phi$ is a Bregman divergence, then, 
by Definition \ref{def:BregmanD}, it follows that
\begin{align*}
\frac d {d\mu} d_\phi(x,\mu) = - \pder{^2 \phi(\mu)}{\mu^2}(x-\mu) 
\end{align*}
Therefore, (\ref{eq:M}) becomes
\begin{equation}
\int (x-\mu) p(x) p(\mu|x) ~dx= 0
\end{equation}
which is equivalent to (\ref{eq:mu_star}) since 
$\int p(x) p(\mu|x) ~dx = p(\mu)$.
\end{proof}

\subsection{Bifurcation Phenomena}

This optimization procedure takes place for decreasing values 
of the temperature coefficient $T$ such that the 
solution maintains minimum free energy
(thermal equilibrium) while gradually lowering the temperature.
Adding to the physical analogy, it is significant that, 
as the temperature is lowered, the
system undergoes a sequence of ``phase transitions'', which
consists of natural cluster splits where the cardinality of the 
codebook (number of clusters) increases. This is a bifurcation phenomenon 
and  provides a useful tool for controlling the size of the clustering model
relating it to the scale of the solution.

At very high temperature ($T\rightarrow\infty$) the optimization yields
uniform association probabilities 
\begin{align*}
p(\mu|x) = \lim_{T\rightarrow\infty} 
			\frac{e^{-\frac{d(x,\mu)}{T}}}
				{\sum_\mu e^{-\frac{d(x,\mu)}{T}}}
		 = \frac 1 K
\end{align*}
and, provided $d:=d_\phi$ is a Bregman divergence,
all the codevectors are located at the same point:
\begin{align*}
\mu = \E{X}
\end{align*}
which is the expected value of $X$ (Theorem \ref{thm:bregman}). 
This is true regardless of the number of codevectors available.
We refer to the number of different codevectors 
resulting from the optimization process as
\textit{effective codevectors}.
These define the cardinality of the codebook, 
which changes as we lower the temperature.
The bifurcation occurs when the solution above a critical temperature $T_c$ is 
no longer the minimum of the free energy $F$ for $T<T_c$. 
A set of coincident codevectors then splits into separate subsets.
These critical temperatures $T_c$ can be traced when the Hessian of $F$ 
loses its positive definite property, and
are, in some cases, computable (see Theorem 1 in \cite{rose1998deterministic}).
In other words, an algorithmic implementation needs only
as many codevectors as the number of effective codevectors, which
depends only on the temperature parameter, i.e. the Lagrange multiplier
of the multi-objective minimization problem in (\ref{eq:F}).
As will be shown in Section \ref{sSec:Algorithm}, 
we can detect the bifurcation points 
by maintaining and perturbing pairs of codevectors at each 
effective cluster so that they separate only when a critical temperature
is reached. 

\subsection{Online Deterministic Annealing for Clustering}

The conditional expectation $\E{X|\mu}$ in eq. (\ref{eq:mu_star})
can be approximated by the sample mean 
of the data points weighted 
by their association 
probabilities $p(\mu|x)$, i.e., 
$$\Eh{X|\mu} = \frac{\sum x p(\mu|x)}{p(\mu)}.$$ 
This approach, however, defines an offline (batch) optimization algorithm 
and requires the entire dataset to be available a priori,
subtly assuming that it is possible to store 
and also quickly access the entire dataset at each iteration. 
This is rarely the case in practical applications and 
results to computationally costly iterations that are slow to converge.
We propose an Online Deterministic Annealing (ODA) algorithm, 
that dynamically updates 
its estimate of the effective codevectors with every observation.
This results in a significant reduction in complexity, 
that comes in two levels.
The first refers to huge reduction in memory complexity, 
since we bypass the need to store the entire dataset, 
as well as the association probabilities 
$\cbra{p(\mu|x),\ \forall x}$ that map each
data point in the dataset to each cluster.
The second level refers to the nature of the optimization iterations.
In the online approach 
the optimization iterations increase in number
but become much faster, and practical convergence is 
often after a smaller number of observations.

To define an online training rule 
for the above optimization framework,
we formulate a stochastic approximation algorithm 
to \emph{recursively estimate $\E{X|\mu}$ directly}.
Stochastic approximation, first introduced in 
\cite{robbins1951stochastic},
was originally conceived as a tool for statistical computation, 
and, since then, has become a central tool in a number of 
different disciplines, 
often times unbeknownst to the users, researchers and practitioners.
It offers an online, adaptive, 
and computationally inexpensive optimization framework, 
properties that make it an ideal optimization method
for machine learning algorithms. 
%
%
In addition to its connection with optimization and learning 
algorithms, however, 
stochastic approximation is strongly connected to dynamical systems, as well,
a property that allows the study of its convergence through the analysis of 
an ordinary differential equation, 
as illustrated in the following theorem:

%
\begin{theorem}[\cite{borkar2009stochastic}, Ch.2]
\label{thm:borkar}
	Almost surely, the sequence $\cbra{x_n}\in S\subseteq\mathbb{R}^d$ 
	generated by the following stochastic approximation scheme:
	\begin{align}
		x_{n+1} = x_n + \alpha(n) \sbra{h(x_n) + M_{n+1}},\ n \geq 0	
	\label{eq:sa}	
	\end{align}
	with prescribed $x_0$, 
	\textit{converges} to a (possibly sample path dependent)
	compact, connected, internally chain transitive, invariant set
	of the o.d.e:
	\begin{align}
		\dot{x}(t) = h\pbra{x(t)}, ~ t \geq 0, 	
	\label{eq:sa_ode}	
	\end{align}
	where $x:\mathbb{R}_+\rightarrow\mathbb{R}_d$ and $x(0) = x_0$, 
	provided the following assumptions hold:
	\begin{itemize}
	\setlength\itemsep{0em}
	\item[(A1)] The map $h:\mathbb{R}^d \rightarrow \mathbb{R}^d$ is Lipschitz
		in $S$,	i.e., $\exists L$ with $0 < L < \infty$ such that
		$\norm{h(x)-h(y)} \leq L\norm{x-y}, ~ x,y \in S$,
	\item[(A2)] The stepsizes $\cbra{\alpha(n) \in \mathbb{R}_{++}, ~ n \geq 0}$
	satisfy
		$ \sum_n \alpha(n) = \infty$, and $\sum_n \alpha^2(n) < \infty$	,
	\item[(A3)] $\cbra{M_n}$ is a martingale difference sequence 
		with respect to the increasing family of $\sigma$-fields
		$ \mathcal{F}_n := \sigma \pbra{ x_m, M_m,~ m \leq n }$, ${n \geq 0}$,
		i.e., $\E{M_{n+1}|\mathcal{F}_n} = 0 ~ a.s.$, for all $n \geq 0$,
		and, furthermore, $\cbra{M_{n}}$ are square-integrable with 
		$ \E{\norm{M_{n+1}}^2|\mathcal{F}_n} \leq K \pbra{ 1 + \norm{x_n}^2 }, 
		~ a.s.$, where $n \geq 0 $ for some $K >0$,
	\item[(A4)] The iterates $\cbra{x_n}$ remain bounded a.s., i.e.,
		${ \sup_n \norm{x_n} < \infty}$ $ a.s.$
	\end{itemize}
\end{theorem}

As an immediate result, the following corollary also holds:
\begin{corollary}
If the only internally chain transitive invariant sets for
(\ref{eq:sa_ode}) are isolated equilibrium points,
then, almost surely, $\cbra{x_n}$ converges to a, 
possibly sample dependent, equilibrium point of (\ref{eq:sa_ode}).  
\label{crl:sa_equillibria}
\end{corollary}

Now we are in place to prove the following theorem:

\begin{theorem}
Let $S$ a vector space, $\mu\in S$, and
$X: \Omega \rightarrow S$
be a random variable defined in a 
probability space $\pbra{\Omega, \mathcal{F}, \mathbb{P}}$.
Let $\cbra{x_n}$ be a sequence of independent realizations of $X$,
and $\cbra{\alpha(n)>0}$ a sequence of stepsizes such that
$ \sum_n \alpha(n) = \infty$, and $\sum_n \alpha^2(n) < \infty$.
Then the random variable $m_n = \nicefrac{\sigma_n}{\rho_n}$,
where $(\rho_n, \sigma_n)$ are 
sequences defined by
\begin{equation}
\begin{aligned}
\rho_{n+1} &= \rho_n + \alpha(n) \sbra{ p(\mu|x_n) - \rho_n} \\
\sigma_{n+1} &= \sigma_n + \alpha(n) \sbra{ x_n p(\mu|x_n) - \sigma_n},
\end{aligned}
\label{eq:rhosigma}
\end{equation}
converges to $\E{X|\mu}$ almost surely, i.e. 
$m_n\xrightarrow{a.s.} \E{X|\mu}$.
\label{thm:oda_sa}
\end{theorem}

\begin{proof}
We will use the facts that $p(\mu)=\E{p(\mu|x)}$ and 
$\E{\mathds{1}_{\sbra{\mu}}X} = \E{xp(\mu|x)}$.
The recursive equations (\ref{eq:rhosigma}) are 
stochastic approximation algorithms of the form:
\begin{equation}
\begin{aligned}
\rho_{n+1} &= \rho_n + \alpha(n)  
	[ (p(\mu) - \rho_n) + \\ 
	&\quad\quad\quad\quad\quad\quad\quad\quad
	(p(\mu|x_n)-\E{p(\mu|X)}) ] \\
\sigma_{n+1} &= \sigma_n + \alpha(n) 
	[ (\E{\mathds{1}_{\sbra{\mu}}X} - \sigma_n) + \\
	&\quad\quad\quad\quad\quad\quad
	(x_n p(\mu|x_n) - \E{x_n p(\mu|X)})  ]
\end{aligned}
\label{eq:rhosigma_sa}
\end{equation}
It is obvious that both stochastic approximation algorithms
satisfy the conditions of 
Theorem \ref{thm:borkar} and Corollary \ref{crl:sa_equillibria}.
As a result, they converge to the asymptotic solution of the 
differential equations
\begin{equation*}
\begin{aligned}
\dot \rho &= p(\mu) - \rho \\
\dot \sigma &= \E{\mathds{1}_{\sbra{\mu}}X} - \sigma
\end{aligned}
\end{equation*}
which can be trivially derived through standard ODE analysis to 
be $\pbra{p(\mu), \E{\mathds{1}_{\sbra{\mu}}X}}$.
In other words, we have shown that
\begin{equation}
\pbra{\rho_n,\sigma_n} \xrightarrow{a.s.} \pbra{p(\mu), \E{\mathds{1}_{\sbra{\mu}}X}}
\end{equation}
The convergence of $m_n$ follows from the fact that 
$\E{X|\mu} = \nicefrac{\E{\mathds{1}_{\sbra{\mu}}X}}{p(\mu)}$,
and standard results on the convergence 
of the product of two random variables.
\end{proof}

As a direct consequence of this theorem, the following corollary 
provides an online learning rule that solves the
optimization problem of the deterministic annealing algorithm.

\begin{corollary}
The online training rule 
\begin{equation}
\begin{cases}
\rho_i(n+1) &= \rho_i(n) + \alpha(n) \sbra{ \hat p(\mu_i|x_n) - \rho_i(n)} \\
\sigma_i(n+1) &= \sigma_i(n) + \alpha(n) \sbra{ x_n \hat p(\mu_i|x_n) - \sigma_i(n)}
\end{cases}
\label{eq:oda_learning1}
\end{equation}
where the quantities $\hat p(\mu_i|x_n)$ and $\mu_i(n)$ 
are recursively updated 
as follows:
\begin{equation}
\begin{aligned}
\hat p(\mu_i|x_n) &= \frac{\rho_i(n) e^{-\frac{d(x_n,\mu_i(n))}{T}}}
			{\sum_i \rho_i(n) e^{-\frac{d(x_n,\mu_i(n))}{T}}} \\
\mu_i(n) &= \frac{\sigma_i(n)}{\rho_i(n)},
\end{aligned}
\label{eq:oda_learning2}
\end{equation}
converges almost surely to a possibly sample path dependent solution of the block optimization 
(\ref{eq:E}), (\ref{eq:mu_star}).
\end{corollary}

Finally, the learning rule
(\ref{eq:oda_learning1}), (\ref{eq:oda_learning2}) can be used to 
define a consistent (histogram)
density estimator at the limit $T\rightarrow 0$. 
%
This follows from the fact that as $T\rightarrow 0$, 
the number of clusters $K$ goes to infinity, 
$p(\mu_h|X)\rightarrow \mathds{1}_{\sbra{X\in S_h}}$, and,
as a result, $F\rightarrow J$, i.e., the consistency of
Alg. \ref{alg:ODA} can be studied with similar arguments 
to the stochastic divergence-based 
vector quantization algorithm (\ref{def:sVQ}) 
(see \cite{mavridis2020convergence,devroye2013}).
%

\subsection{Online Deterministic Annealing for Classification}

We can extend the proposed learning algorithm 
to be used for classification as well.
In this case we can rewrite the expected distortion as
\begin{align*}
D = \E{d^b(c_X,Q^c)} 
\end{align*}
where $ d^b(c_x,c_\mu) = \mathds{1}_{\sbra{c_x\neq c_\mu}}$.
%
Because $d^b$ is not differentiable, 
using similar principles as in the case of LVQ, 
we can instead approximate the optimal solution by 
solving the minimization 
problem for the following distortion measure
\begin{equation}
d^c(x,c_x,\mu,c_\mu) = \begin{cases}
				d(x,\mu),~ c_x=c_\mu \\
				0,~ c_x\neq c_\mu			
				\end{cases}
\label{eq:dc}
\end{equation}
This particular choice for the distortion measure $d^c$ 
will lead to some interesting regularization properties 
of the proposed online approach (see Section \ref{sSec:Algorithm}). 

It is easy to show that the coordinate block optimization steps 
(\ref{eq:E}) and (\ref{eq:mu_star}), in this case become:
%
\begin{align*}
p(\mu,c_\mu|x,c_x) = \frac{e^{-\frac{d^c(x,c_x,\mu,c_\mu)}{T}}}
			{\sum_{\mu,c_\mu} e^{-\frac{d^c(x,c_x,\mu,c_\mu)}{T}}},\
			\text{and} \\
\mu^* = \frac{\sum_{c_x=c_\mu} x p(x,c_x) p(\mu,c_\mu|x,c_x)}
 {\sum_{c_x=c_\mu} p(x,c_x) p(\mu,c_\mu|x,c_x)}
\end{align*}
%
%
%
respectively. 
In the last step, we have assumed that the class $c_\mu$ 
of each centroid $\mu$ is given and cannot be
changed dynamically by the algorithm, 
which results to the minimization with respect
to $\mu$ only. 
In a similar fashion, 
it can be shown that the 
online learning rule that solves the
optimization problem of the deterministic annealing algorithm
for classification, based on the distortion measure 
(\ref{eq:dc}), is given by:

\begin{equation}
\begin{aligned}
\rho_i(n+1) =& \rho_i(n) + \alpha(n) \mathds{1}_{\sbra{c_{x_j}=c_{\mu_i}}} \\ 
&  \sbra{\hat p(\mu_i,c_{\mu_i}|x_n,c_{x_n}) - \rho_i(n)} \\
\sigma_i(n+1) =& \sigma_i(n) + \alpha(n) \mathds{1}_{\sbra{c_{x_j}=c_{\mu_i}}} \\ 
& \sbra{x_n \hat p(\mu_i,c_{\mu_i}|x_n,c_{x_n}) - \sigma_i(n)}
\end{aligned}
\label{eq:oda_learning1c}
\end{equation}
where 
%
\begin{equation}
\begin{aligned}
\hat p(\mu_i,c_{\mu_i}|x_n,c_{x_n}) &= \frac{\rho_i(n) e^{-\frac{d^c(x_n,c_{x_n},\mu_i(n),c_{\mu_i(n)})}{T}}}
			{\sum_i \rho_i(n) e^{-\frac{d^c(x_n,c_{x_n},\mu_i(n),c_{\mu_i(n)})}{T}}} \\
\mu_i(n) &= \frac{\sigma_i(n)}{\rho_i(n)}
\end{aligned}
\label{eq:oda_learning2c}
\end{equation}
At the limit $T\rightarrow 0$, 
the quantization scheme described above 
equipped with a majority-vote classification rule is strongly Bayes 
risk consistent, i.e., converges to the optimal (Bayes) probability 
of error (see Ch. 21 in \cite{devroye2013}). 
%
However, due to the choice of the distortion measure $d^c$ in 
(\ref{eq:dc}) used in ODA for classification, 
the algorithm can be used 
to estimate consistent class-conditional density estimators, 
which define 
the natural classification rule:
\begin{equation}
\hat c(x) = c_{\mu_{h^*}}
\end{equation}
where $h^* = \argmax\limits_{\tau = 
		1,\ldots,K} ~ p(\mu_\tau|x),~ h \in \cbra{1, \ldots, K}$.

\subsection{The algorithm}
\label{sSec:Algorithm}

The proposed Online Deterministic Annealing (ODA) algorithm %
(Algorithm \ref{alg:ODA}),  
is based on (\ref{eq:oda_learning1c}), (\ref{eq:oda_learning2c}), and
can be used for both clustering and classification alike,
depending on whether the data
belong to a single (clustering) or several classes (classification).
%

\textit{Temperature Schedule.}
The temperature schedule $T_i$ plays an important role in 
the behavior of the algorithm.
Starting at high temperature
$T_{max}$ ensures the correct operation of the algorithm.
The value of $T_{max}$ depends on the domain of the data 
and should be large enough such that there is only one effective codevector at $T=T_{max}$. 
When the range of the domain of the data is not known a priori, 
overestimation is recommended. 
The stopping temperature $T_{min}$ can be set a priori or be decided
online depending on the performance of the model at each temperature level. 
The temperature step $dT_i=T_{i-1}-T_{i}$ should be small enough such 
that no critical temperature is missed. 
On the other hand, the smaller the step $dT_i$,
the more optimization problems need to be solved.
It is common practice to use the geometric series $T_{i+1}=\gamma T_i$. 
%

\textit{Stochastic Approximation.}
Regarding the stochastic approximation stepsizes, 
simple time-based learning rates, 
e.g. of the form $\alpha_n = \nicefrac{1}{a+ bn}$, 
have been sufficient 
for fast convergence in all our experiments so far.
Convergence is checked with the condition
$d_\phi(\mu_i^n,\mu_i^{n-1})<\epsilon_c$
for a given threshold $\epsilon_c$ that can depend on the domain of $X$.
Exploring adaptive learning rates would be an interesting 
research direction for the future.

\textit{Bifurcation and Perturbations.}
To every temperature level $T_i$, corresponds a set of 
effective codevectors $\cbra{\mu_j}_{j=1}^{K_i}$, 
which consist of the different solutions 
of the optimization problem (\ref{eq:F}) at $T_i$.
Bifurcation, at $T_i$, is detected by maintaining a pair of perturbed 
codevectors $\cbra{\mu_j+\delta, \mu_j-\delta}$ 
for each effective codevector $\mu_j$ generated at $T_{i-1}$,
i.e. for $j=1\ldots,K_{i-1}$.
Using arguments from variational calculus \cite{rose1998deterministic},
it is easy to see that, upon convegence, 
the perturbed codevectors will merge if a critical 
temperature has not been reached, and will get separated otherwise. 
In case of a merge, one of the perturbed codevectors
is removed from the model. 
Therefore, the cardinality of the model is at most doubled at 
every temperature level.
For classification, a perturbed codevector for each distinct class is generated.
%

\textit{Regularization.}
Merging is detected by the condition $d_\phi(\mu_j,\mu_i)<\epsilon_n$,
where $\epsilon_n$ is a design parameter
that acts as a regularization term for the model. 
Large values for $\epsilon_n$ (compared to the support of the data $X$)
lead to fewer effective codevectors, while small $\epsilon_n$ 
values lead to a fast growth in the model size, 
which is connected to overfitting.
It is observed that, for practical convergence, 
the perturbation noise $\delta$ is best to not exceed $\epsilon_n$.
An additional regularization mechanism that comes as a natural 
consequence of the stochastic approximation learning rule,
is the detection of idle codevectors.
To see that, notice that the sequence $\rho_i(n)$ resembles an 
approximation of the probability $p(\mu_i,c_{\mu_i})$.
In the updates (\ref{eq:oda_learning1}), (\ref{eq:oda_learning2}),
$\rho_i(n)$ 
becomes negligible ($\rho_i(n)<\epsilon_r$) 
if not updated by any nearby observed data,
which is a natural criterion for removing the codevector $\mu_i$.
This happens if all observed data samples $x_n$
are largely dissimilar to $\mu_i$. 
In classification, because of the choice of $d^c$ in (\ref{eq:dc}),
codevectors $\mu_i$ that are not assigned the same class as the data in 
their vicinity, will end up to be removed, as well.
The threshold $\epsilon_r$ is a parameter that usually takes 
values near zero.

\textit{Complexity.}
The worst case complexity of Algorithm \ref{alg:ODA} behaves as 
$O(\sigma_{max} N K_{max}^2 d),$
where:
%
\begin{itemize}
\item $N$ is an upper bound of the number of data samples observed,
	which should be large enough to overestimate
	the iterations needed for convergence;
\item $d$ is the dimension of the input vectors, i.e., $x\in\mathbb{R}^d$;
\item $K_{max}$ is the maximum number of codevectors allowed; 
\item $\sigma_{max} = \cbra{\sigma_1,\sigma_2,\ldots,\sigma_{K_{max}}}$,
	where $\sigma_i$ is the number of temperature values in our 
	temperature schedule that lie between two critical temperatures
	$T_i$ and $T_{i+1}$, with the understanding that at $T_i$ 
	there are $i$ distinct effective codevectors present. 
	Here we have assumed that $K_{max}$ is achievable within 
	our temperature schedule.  
\end{itemize}

\textit{Fine-Tuning.}
In practice, 
because the convergence to the Bayes decision surface 
comes at the limit $(K,T)\rightarrow (\infty,0)$, 
a fine-tuning mechanism should be designed to run on top of 
the proposed algorithm after $T_{min}$.
This can be either an LVQ algorithm (Section \ref{sSec:LVQ})
or some other local model. 
%


\begin{algorithm}[hb!]
\caption{Online Deterministic Annealing}
\label{alg:ODA}
\begin{algorithmic}
\STATE Select Bregman divergence $d_\phi$
\STATE Set temperature schedule: $T_{max}$, $T_{min}$, $\gamma$
\STATE Decide maximum number of codevectors $K_{max}$ 
\STATE Set convergence parameters: $\cbra{\alpha_n}$, 
	$\epsilon_c$, $\epsilon_n$, $\epsilon_r$, $\delta$ 
\STATE Select initial configuration 
	$ \cbra{\mu^i}: c_{\mu^i} = c,~ \forall c \in\mathcal{C} $
\STATE Initialize: $K = 1$, $T = T_{max}$
\STATE Initialize:	$p(\mu^i) = 1$, $\sigma(\mu^i) = \mu^i p(\mu_i)$, 
	$\forall i$ 	
\WHILE{$K<K_{max}$ \textbf{and} $T>T_{min}$}
\STATE Perturb  
	$\mu^i \gets  
		\cbra{\mu^i+\delta, \mu^i-\delta}$, $\forall i$ 
\STATE Increment $K\gets 2K$
\STATE Update $p(\mu^i)$, $\sigma(\mu^i)\gets\mu^i p(\mu^i)$, $\forall i$ 
\STATE Set $n \gets 0$
\REPEAT 
%
\STATE Observe data point $x$ and class label $c$
\FOR{$i = 1,\ldots, K$} 
\STATE Compute membership $s^i = \mathds{1}_{\sbra{c_{\mu^i}=c}}$ 
\STATE Update: 
\vskip -0.3in
	\begin{align*}
	p(\mu^i|x) &\gets \frac{p(\mu^i) e^{-\frac{d_\phi(x,\mu^i)}{T}}}
			{\sum_i p(\mu^i) e^{-\frac{d_\phi(x,\mu^i)}{T}}} \\
	p(\mu^i) &\gets p(\mu^i) + \alpha_n \sbra{s^i p(\mu^i|x) - p(\mu^i)} \\
	\sigma(\mu^i) &\gets \sigma(\mu^i) + 
		\alpha_n \sbra{s^i x p(\mu^i|x) - \sigma(\mu^i)} \\
	\mu^i &\gets \frac{\sigma(\mu^i)}{p(\mu^i)}	
	\end{align*}
\vspace{-1.5em}
\STATE Increment $n\gets n+1$
\ENDFOR
\UNTIL $d_\phi(\mu^i_n,\mu^i_{n-1})<\epsilon_c$, $\forall i $
\STATE Keep effective codevectors: \\ 
	\quad\quad discard $\mu^i$ if $d_\phi(\mu^j,\mu^i)<\epsilon_n$, 
	$\forall i,j,i\neq j$
\STATE Remove idle codevectors: \\ 
	\quad\quad 	discard $\mu^i$ if $p(\mu^i)<\epsilon_r$, $\forall i$
\STATE Update $K$, $p(\mu^i)$, $\sigma(\mu^i)$, $\forall i$
\STATE Lower temperature $T \gets \gamma T$ 
\ENDWHILE
\end{algorithmic}
\end{algorithm}

\section{Experimental Evaluation and Discussion}
\label{Sec:Results}

We illustrate the properties and evaluate the performance 
of the proposed algorithm in widely used artificial and real datasets
for clustering and classification%
\footnote{Code and Reproducibility: The source code is publicly available online
at \textit{https://github.com/MavridisChristos/OnlineDeterministicAnnealing}.}.

\begin{figure}[t]
\centering
\begin{subfigure}[b]{0.48\textwidth}
\centering
\includegraphics[trim=80 50 60 55,clip,width=0.24\textwidth]{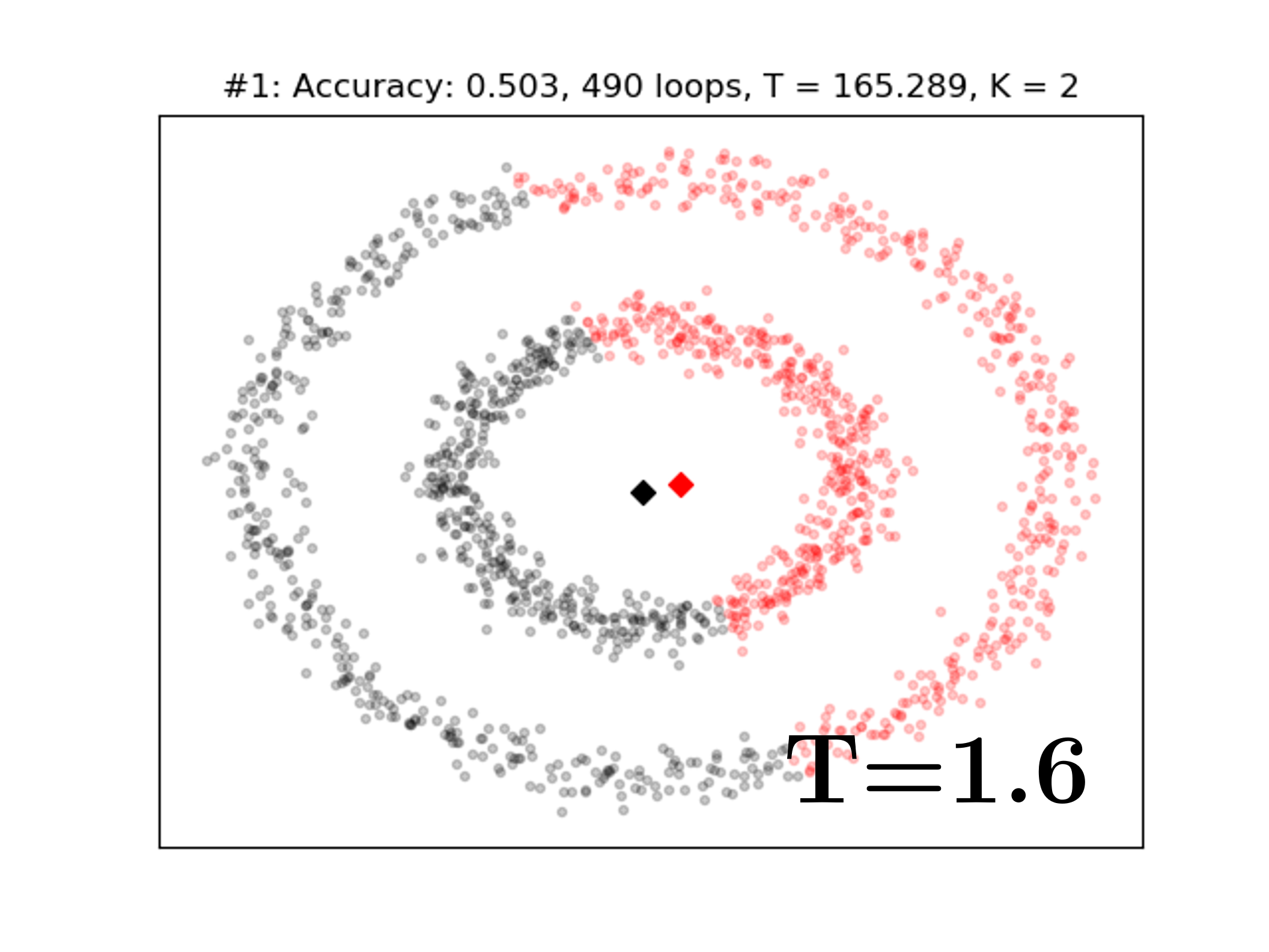}
\includegraphics[trim=80 50 60 55,clip,width=0.24\textwidth]{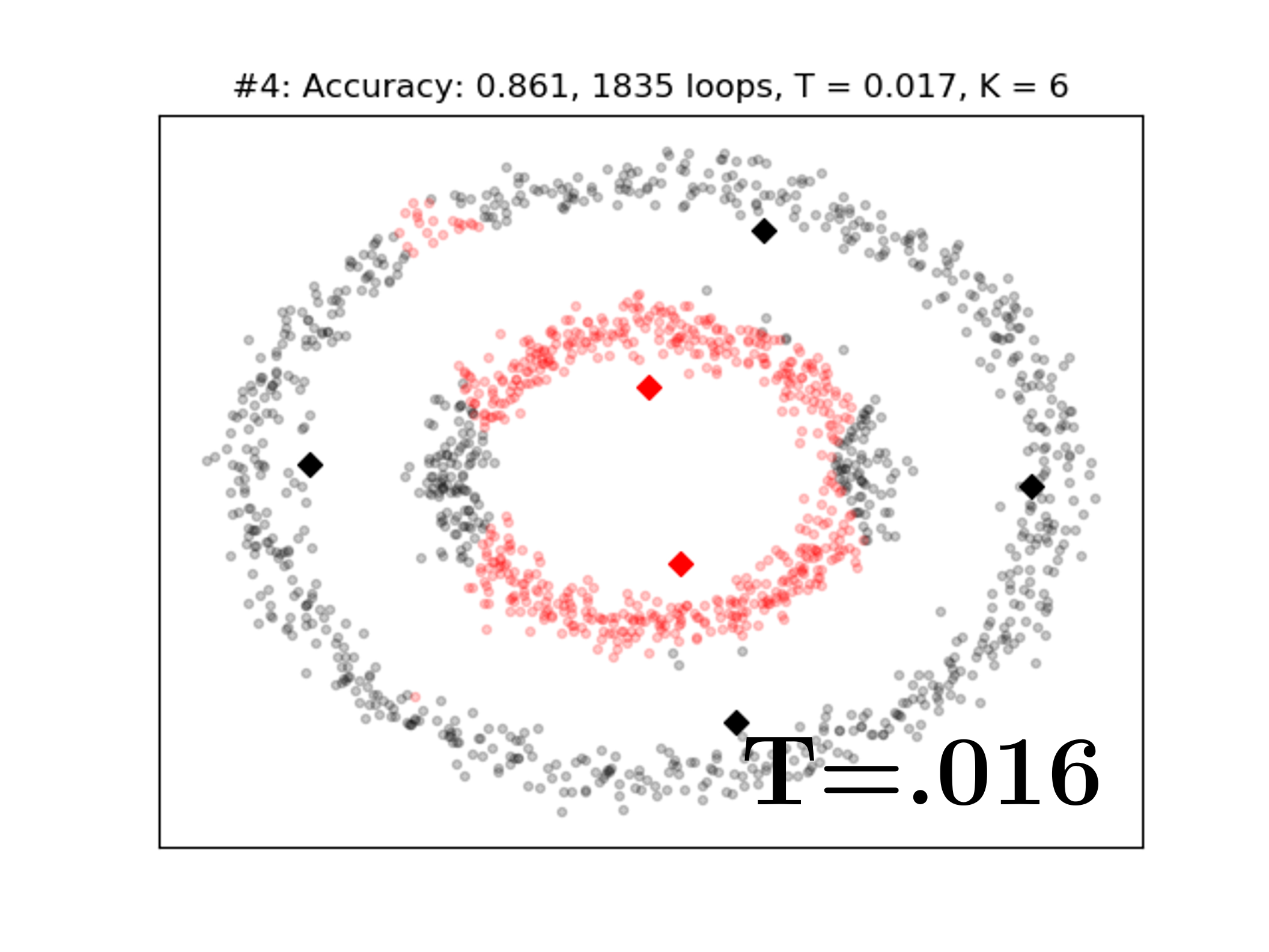}
\includegraphics[trim=80 50 60 55,clip,width=0.24\textwidth]{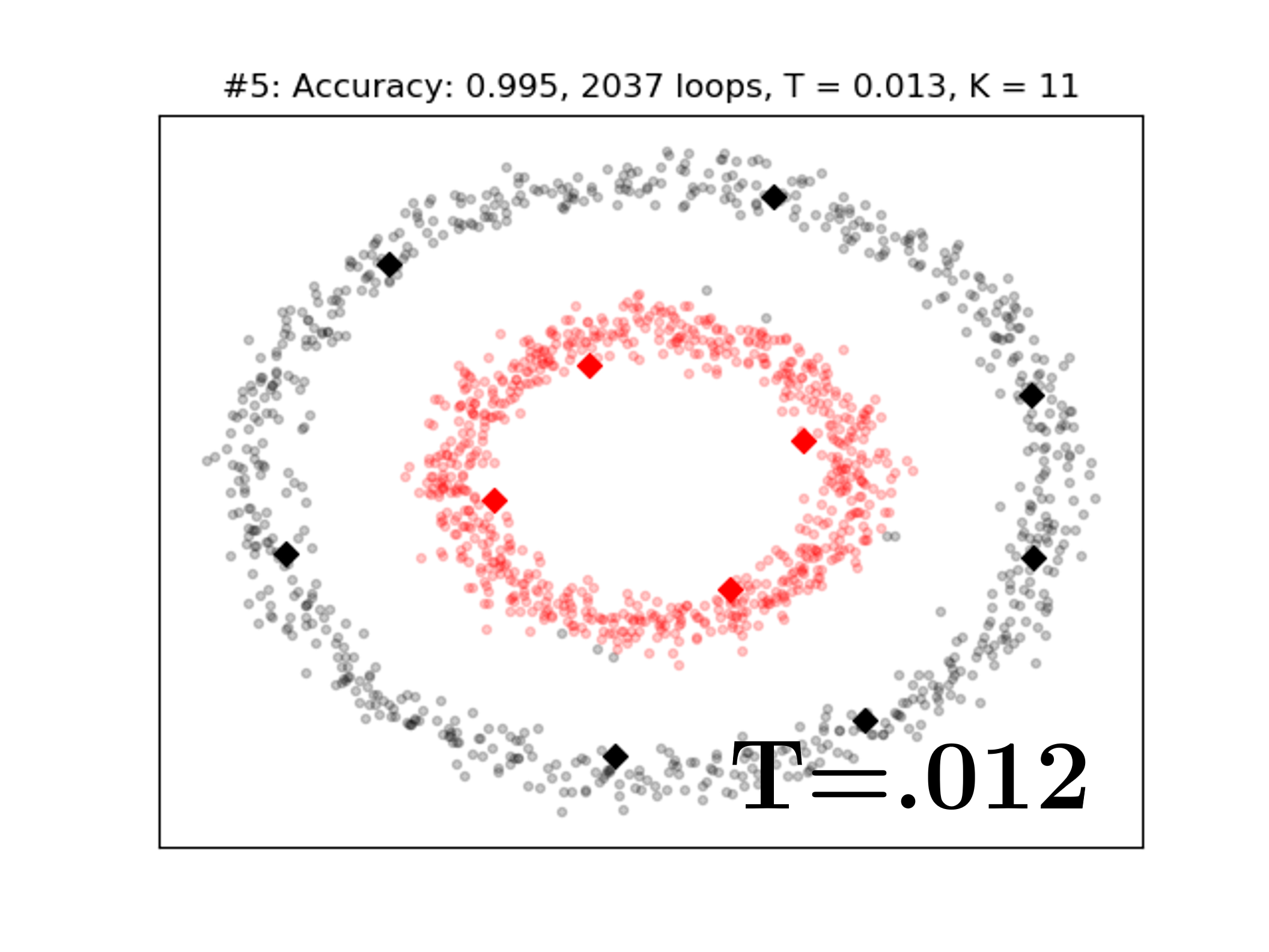}
\includegraphics[trim=80 50 60 55,clip,width=0.24\textwidth]{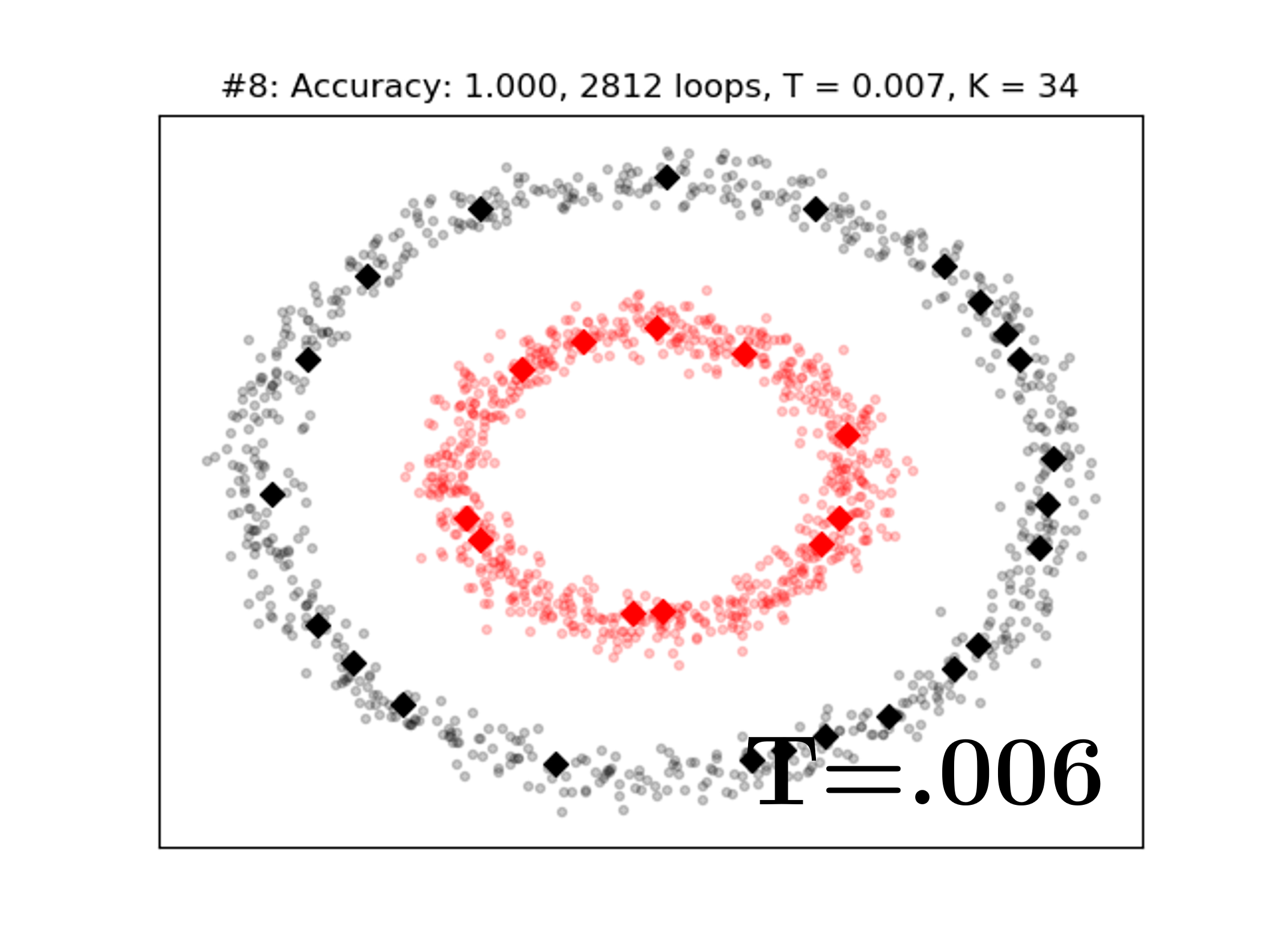}
\caption{Concentric circles.}
\label{sfig:illustration_circles}
\end{subfigure}
\begin{subfigure}[b]{0.48\textwidth}
\centering
\includegraphics[trim=80 50 60 55,clip,width=0.24\textwidth]{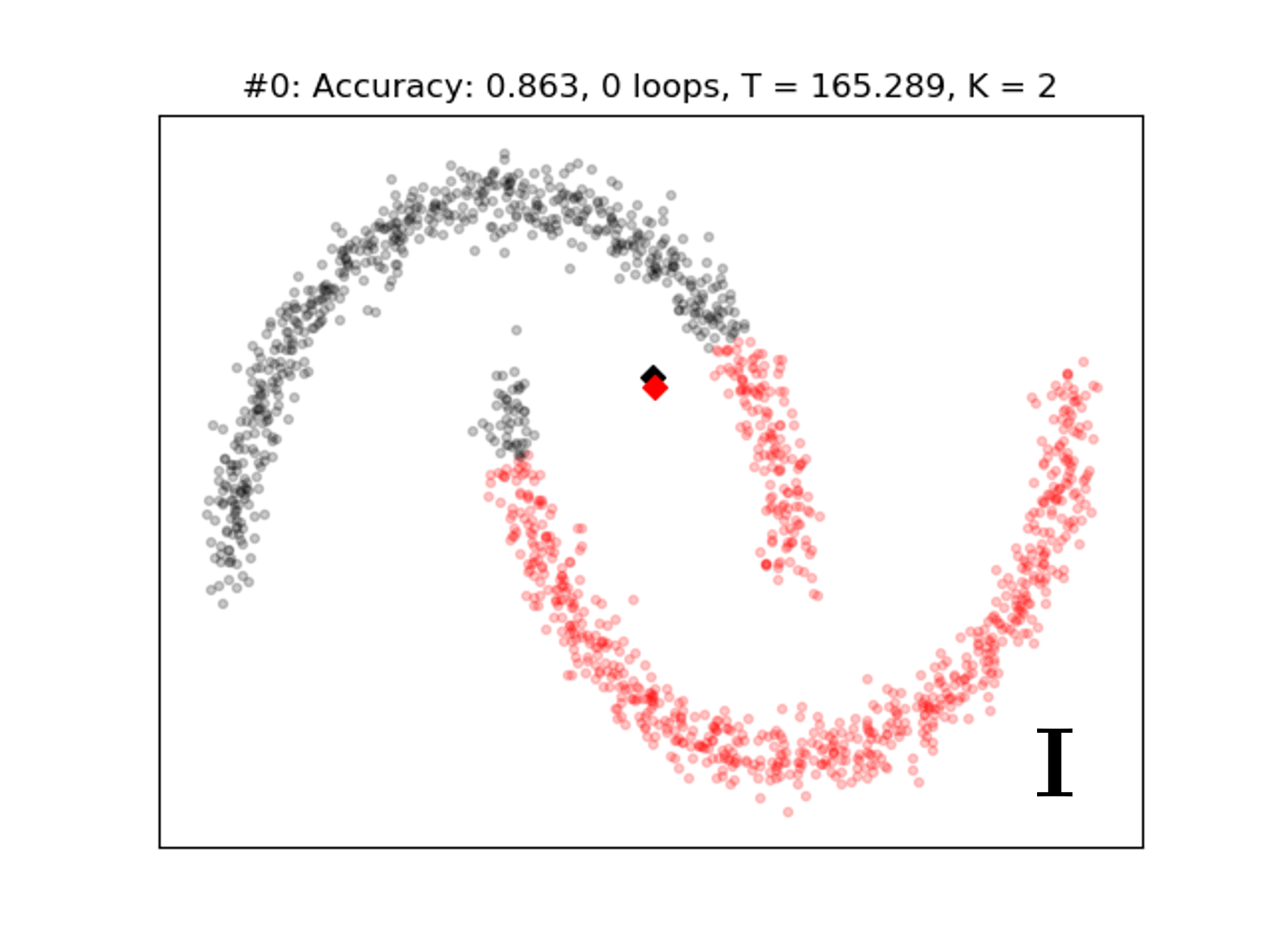}
\includegraphics[trim=80 50 60 55,clip,width=0.24\textwidth]{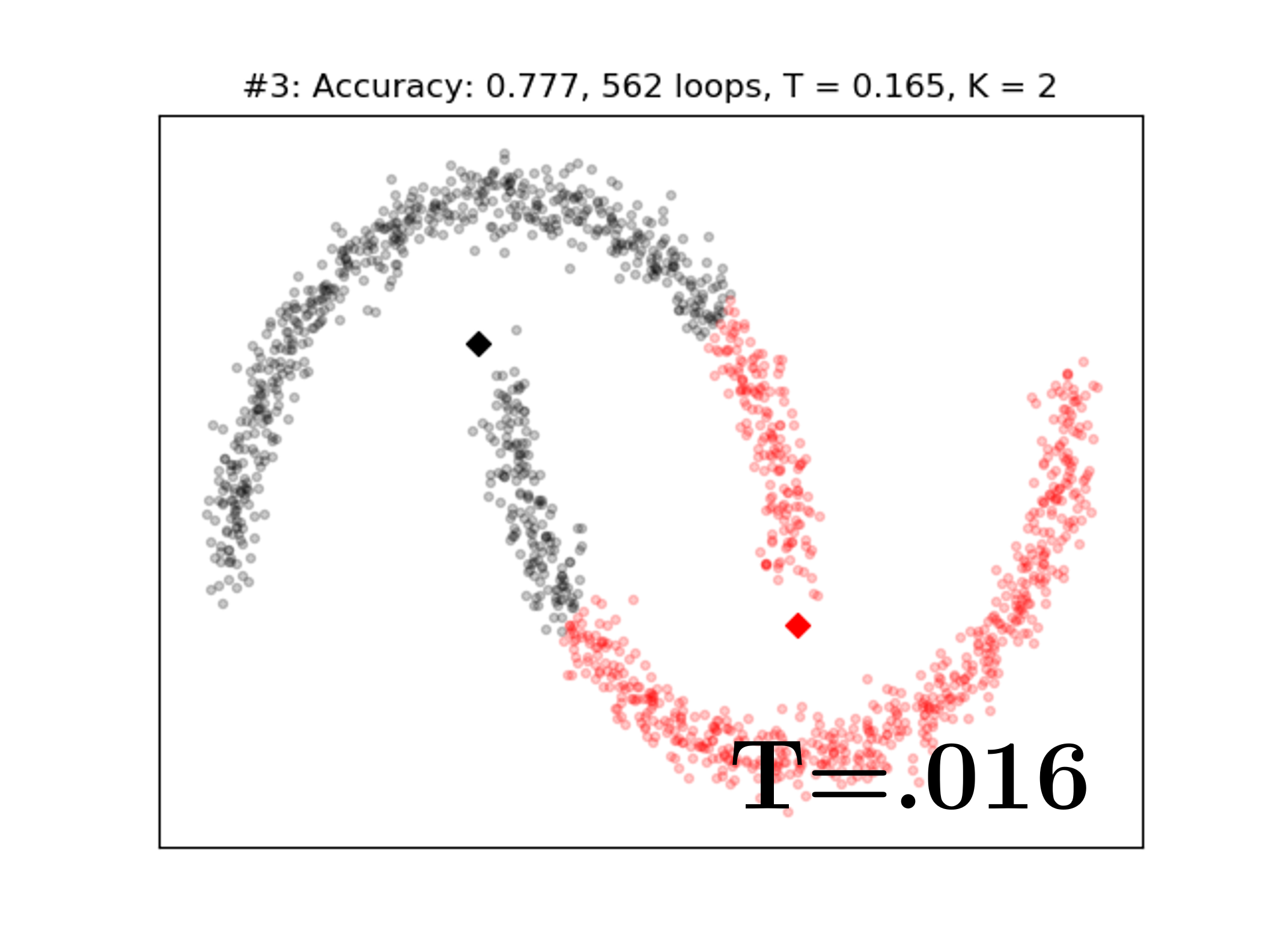}
\includegraphics[trim=80 50 60 55,clip,width=0.24\textwidth]{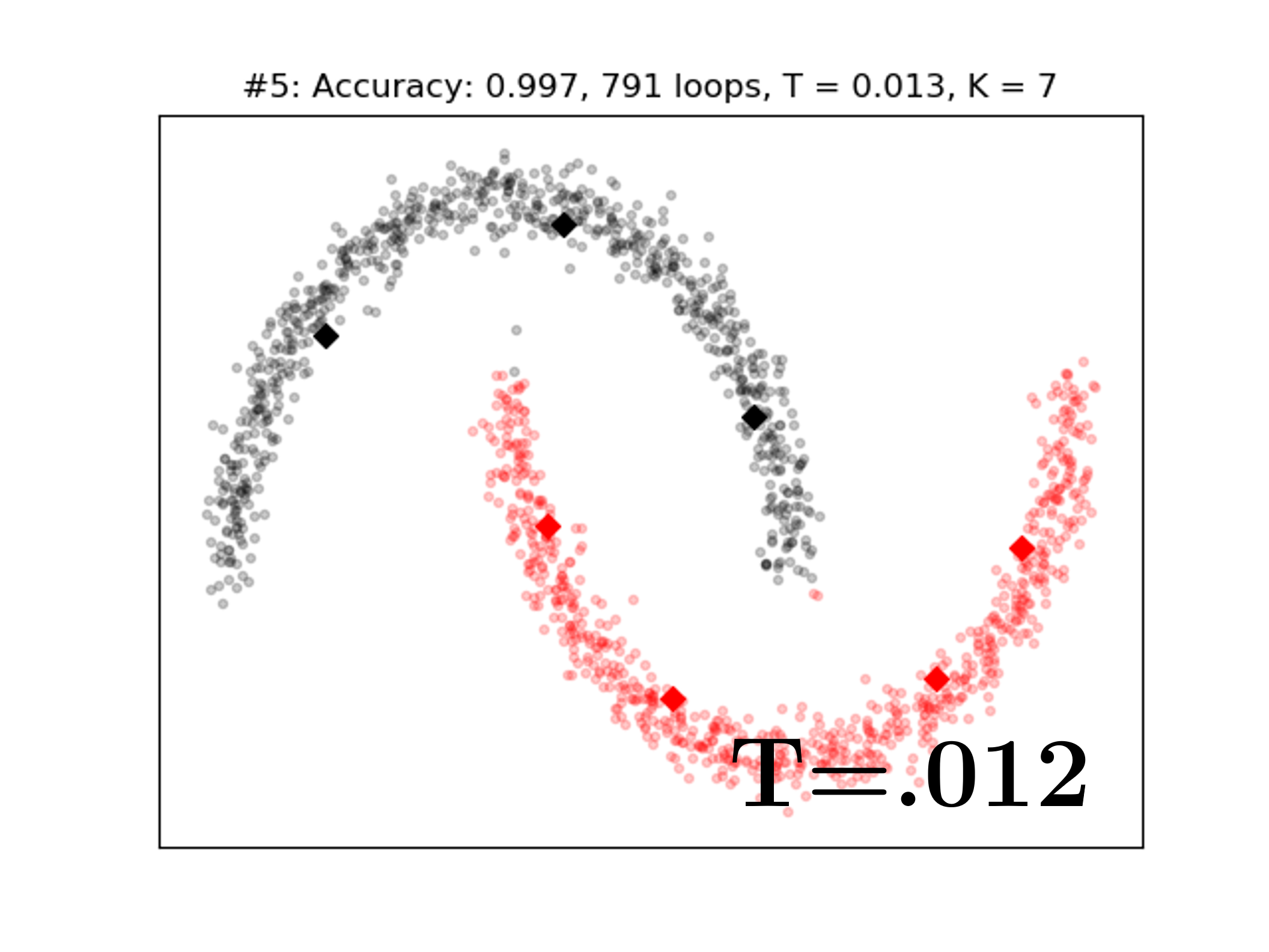}
\includegraphics[trim=80 50 60 55,clip,width=0.24\textwidth]{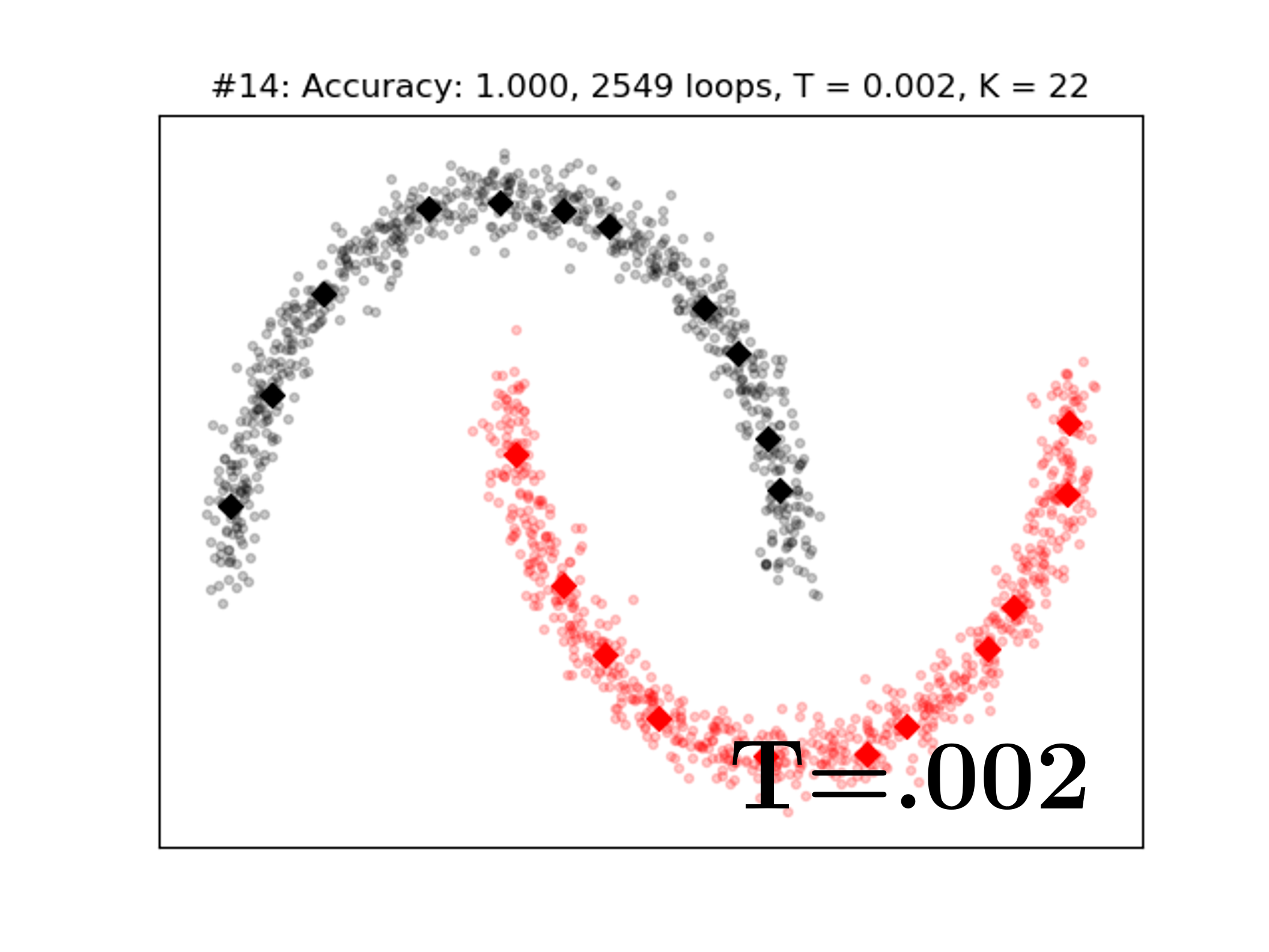}
\caption{Half moons.}
\label{sfig:illustration_moons}
\end{subfigure}
\begin{subfigure}[b]{0.48\textwidth}
\centering
\includegraphics[trim=80 50 60 55,clip,width=0.24\textwidth]{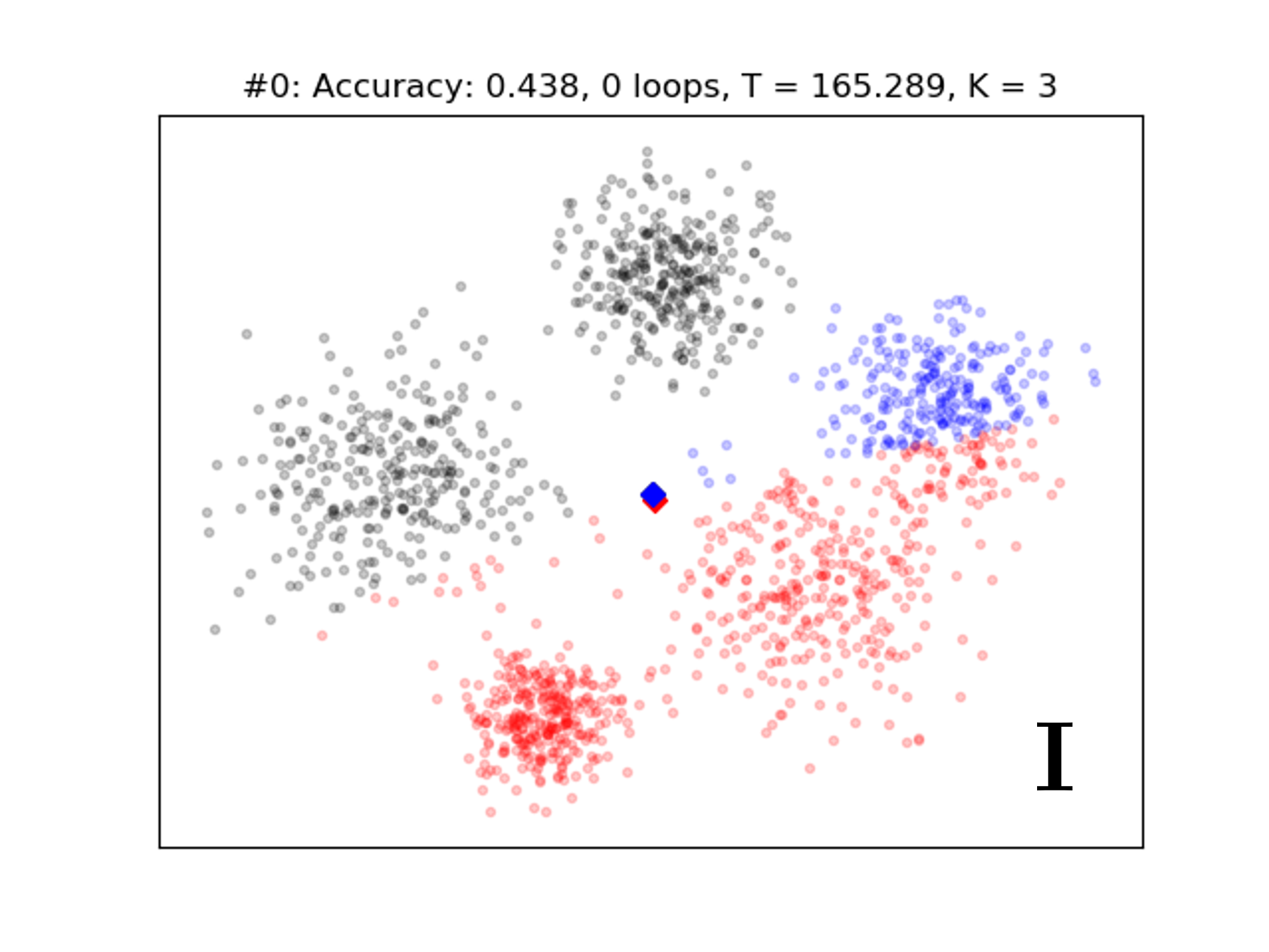}
\includegraphics[trim=80 50 60 55,clip,width=0.24\textwidth]{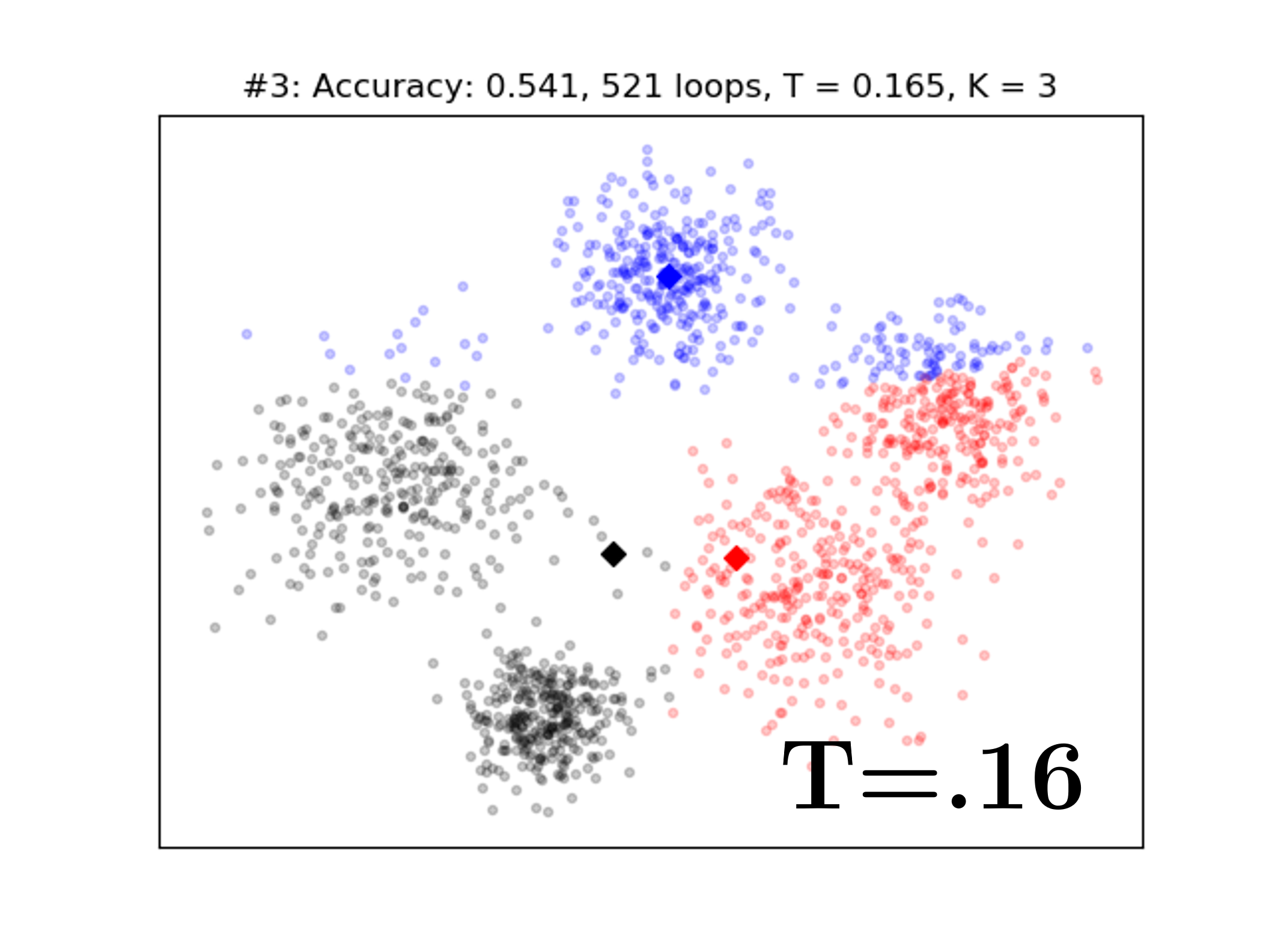}
\includegraphics[trim=80 50 60 55,clip,width=0.24\textwidth]{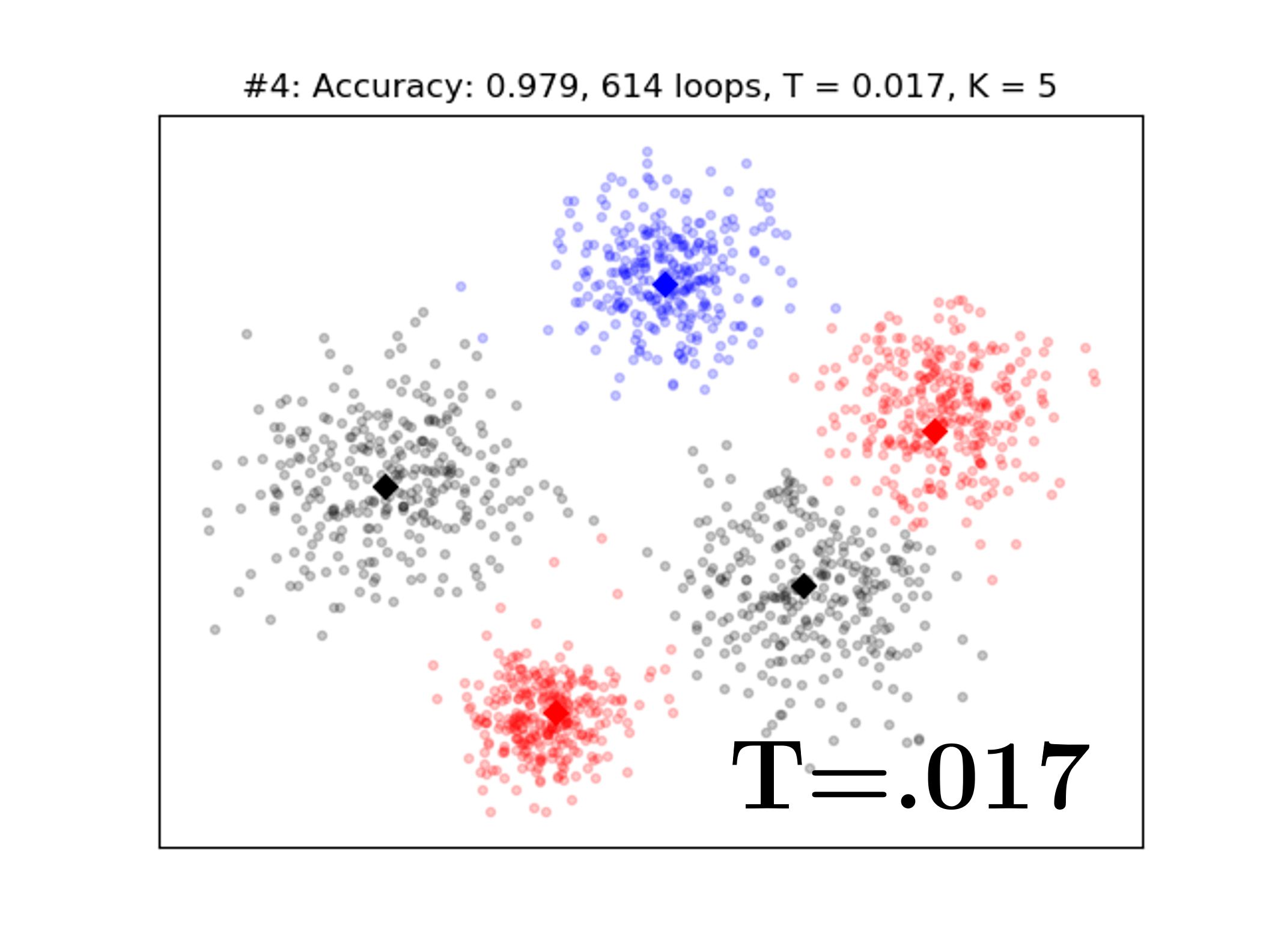}
\includegraphics[trim=80 50 60 55,clip,width=0.24\textwidth]{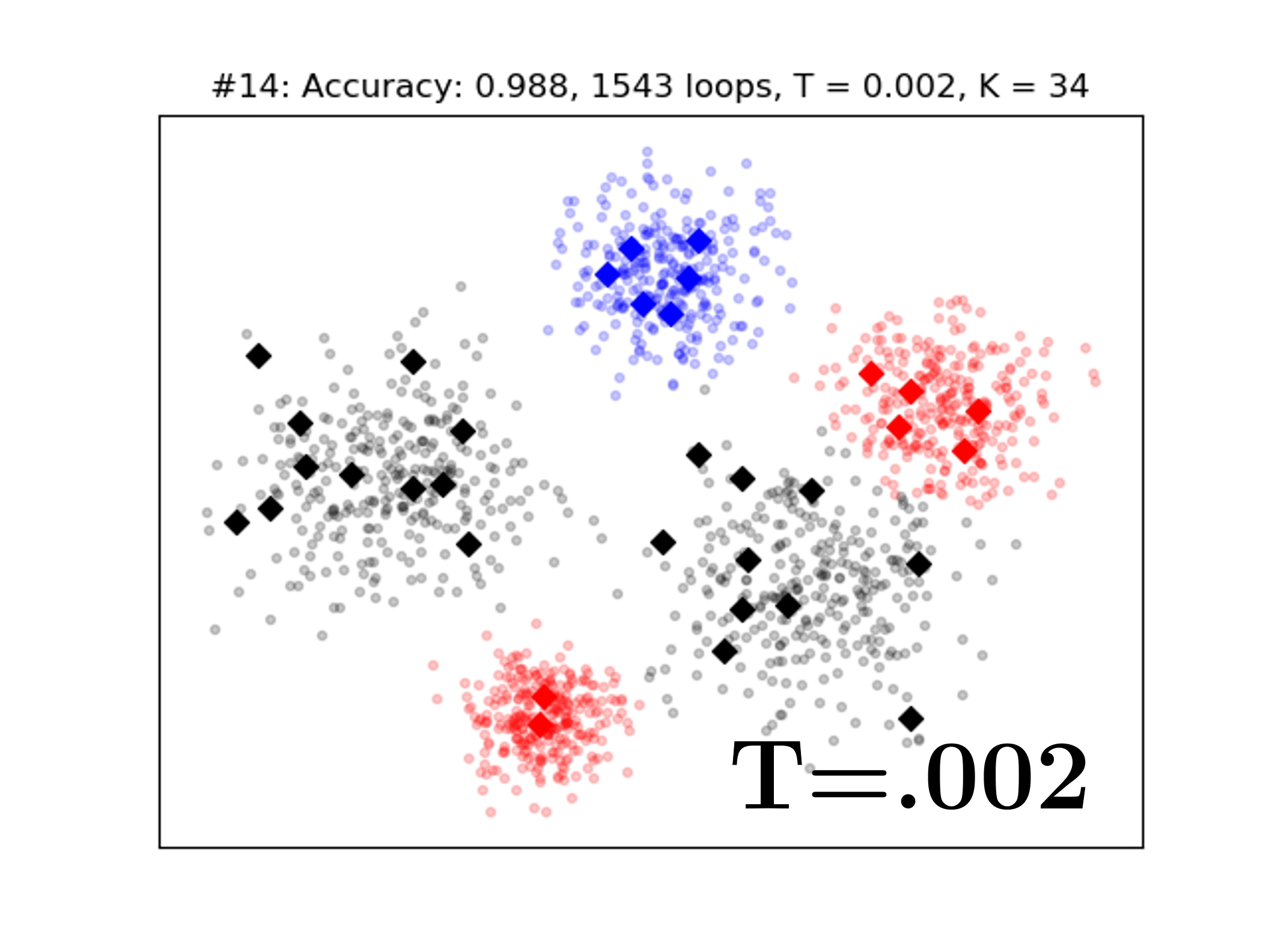}
\caption{Gaussians.}
\label{sfig:illustration_blobs}
\end{subfigure}
\begin{subfigure}[b]{0.48\textwidth}
\centering
\includegraphics[trim=70 45 60 55,clip,width=0.24\textwidth]{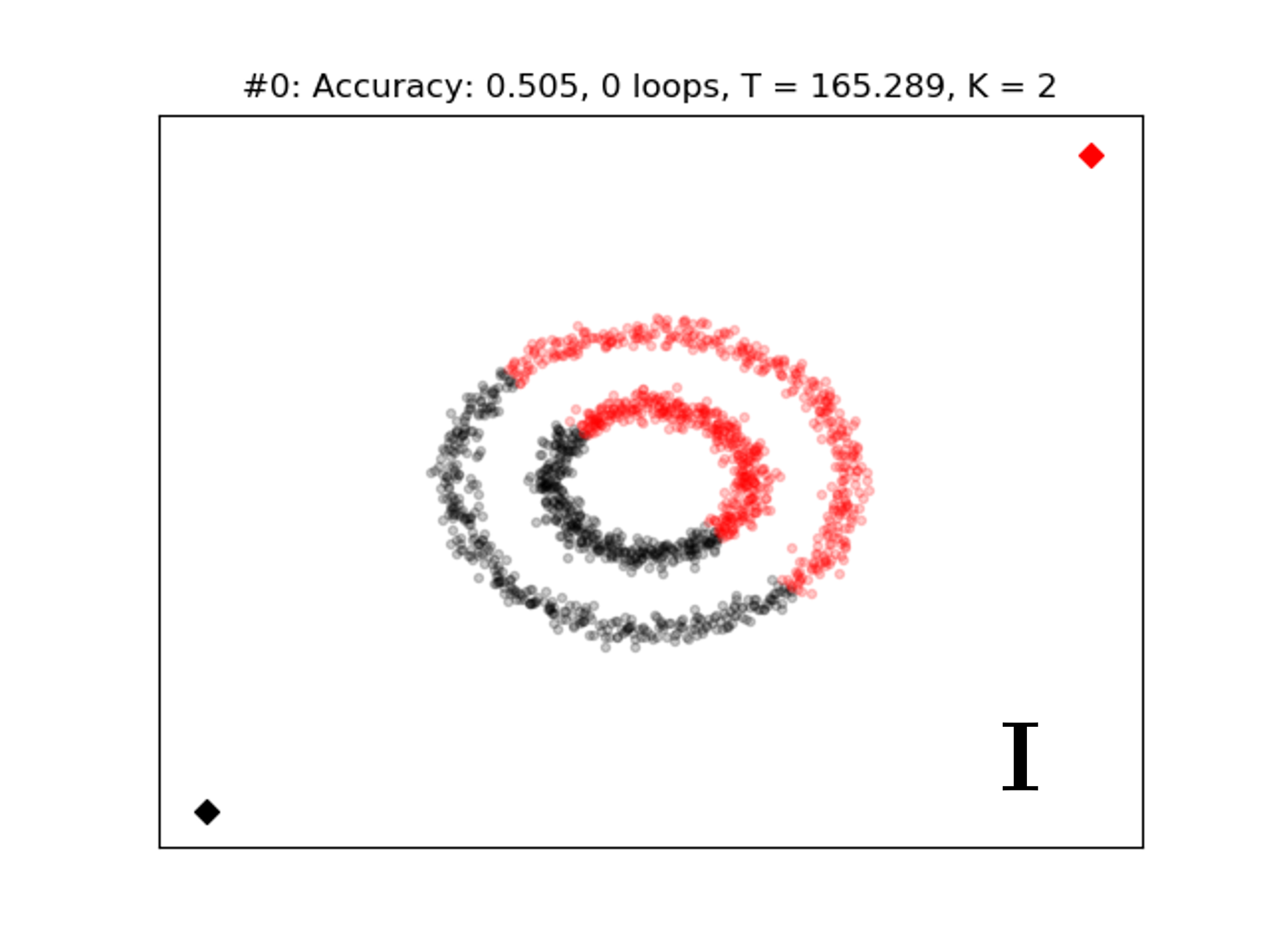}
\includegraphics[trim=70 45 60 55,clip,width=0.24\textwidth]{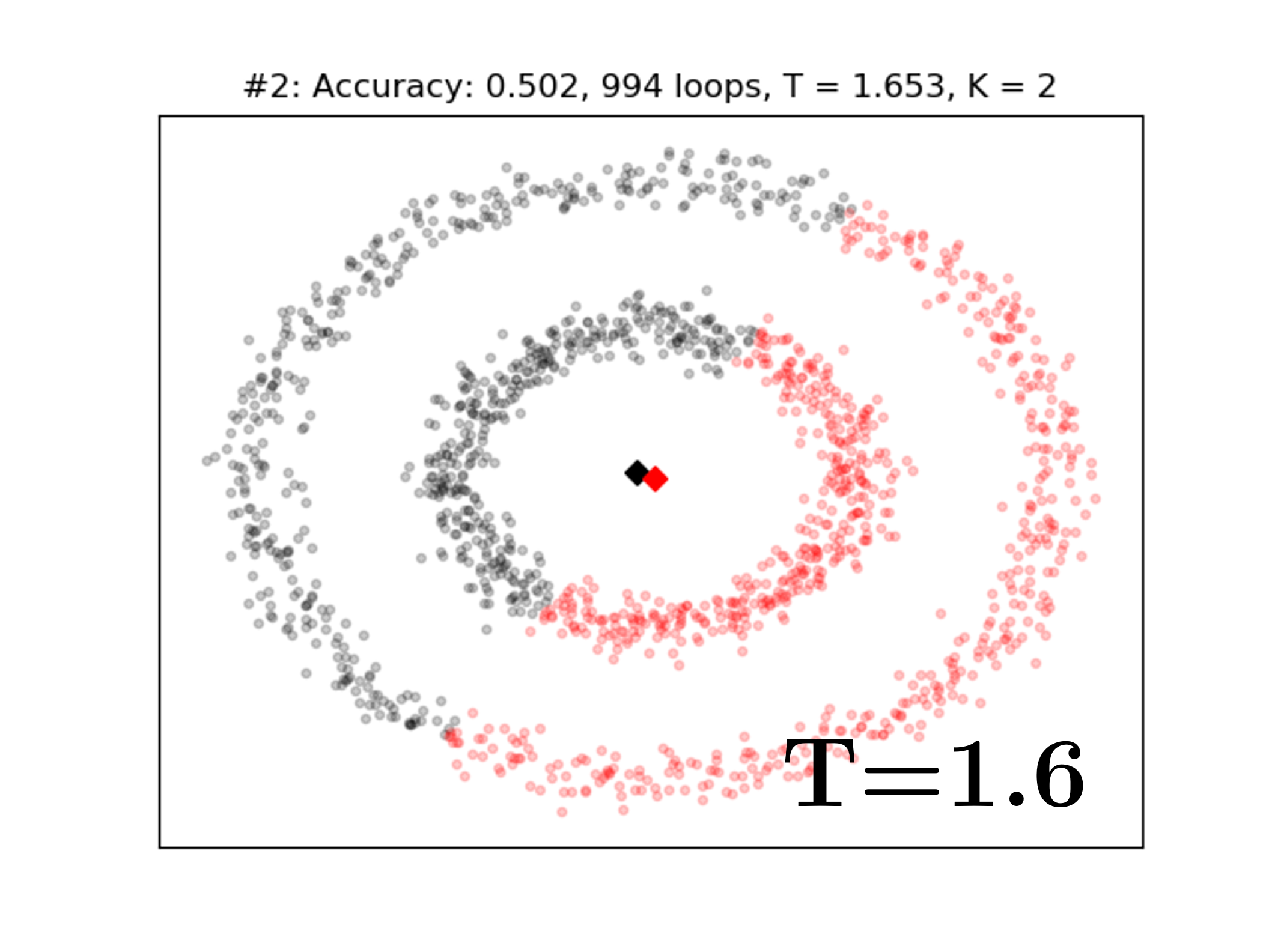}
\includegraphics[trim=70 45 60 55,clip,width=0.24\textwidth]{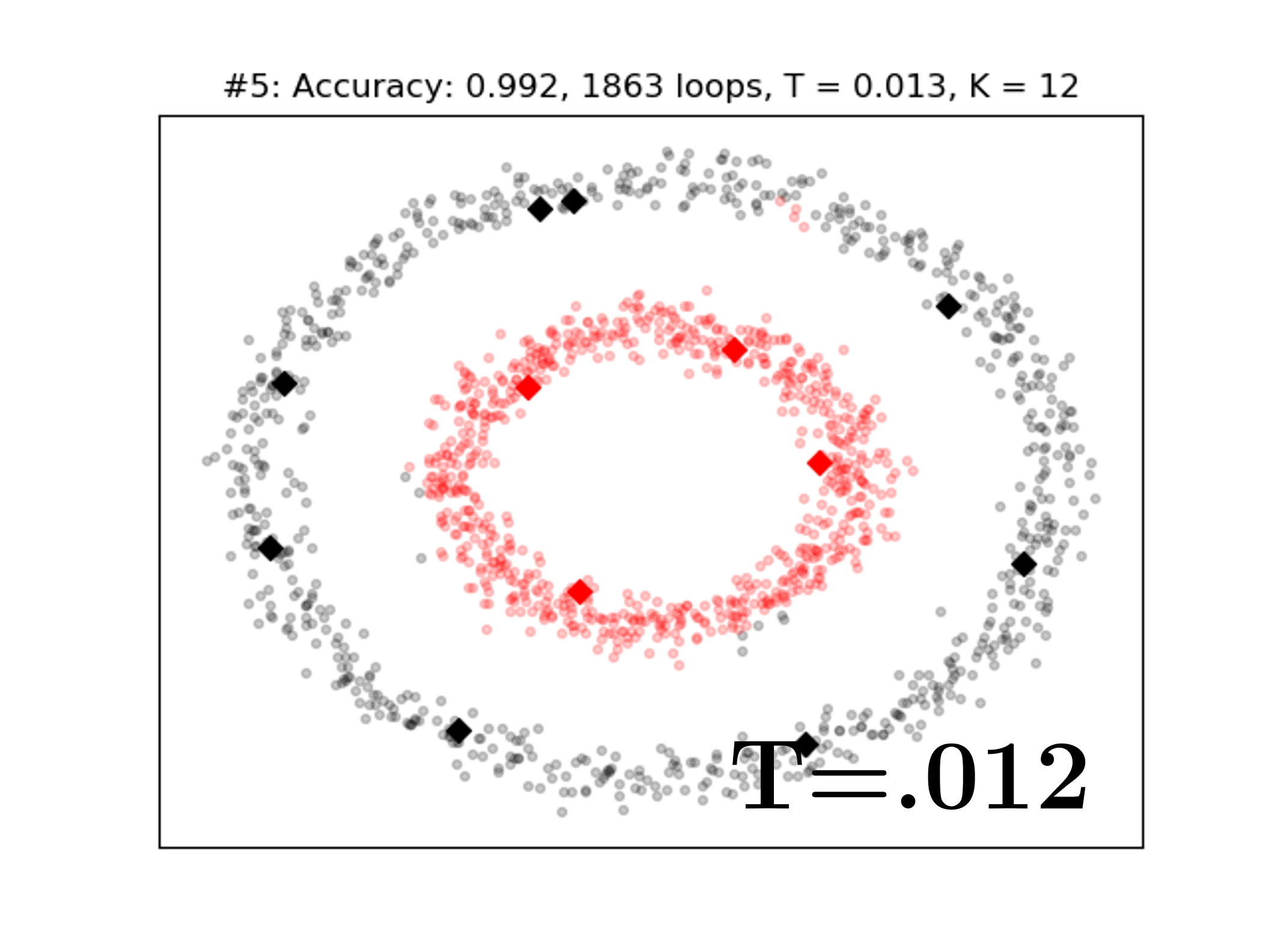}
\includegraphics[trim=70 45 60 55,clip,width=0.24\textwidth]{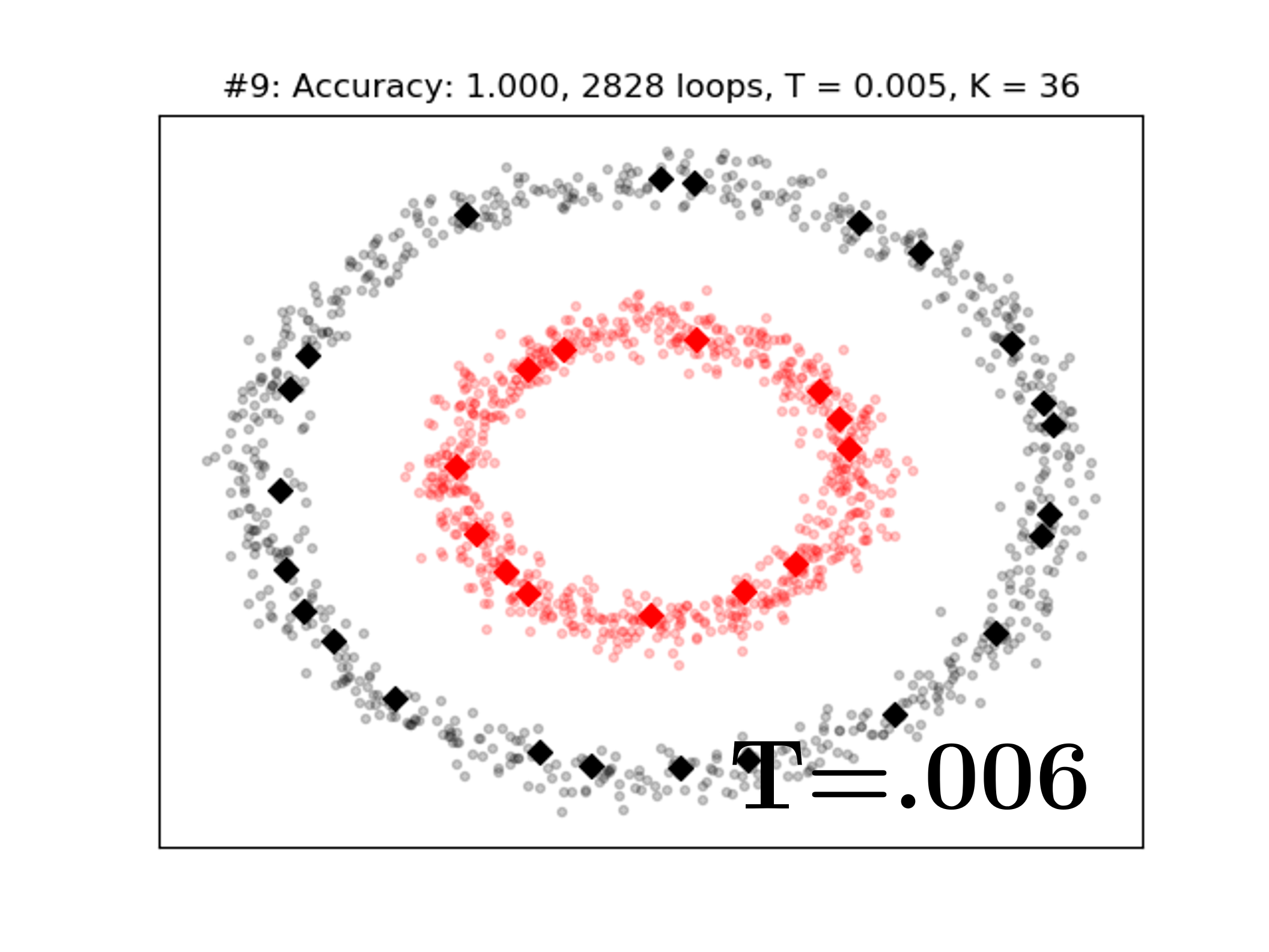}
\caption{Poor initial conditions.}
\label{sfig:illustration_initial}
\end{subfigure}
\caption{$(a)$-$(c)$Illustration of the evolution of Alg. \ref{alg:ODA} 
	for decreasing temperature $T$ in binary classification in 2D.
	$(d)$ Showcasing robustness with respect to bad initial conditions.}
\label{fig:illustration}
\vspace{-1.0em}
\end{figure}
%

\subsection{Toy Examples}

We first showcase how Alg. \ref{alg:ODA} 
works in three simple, but illustrative, classification problems
in two dimensions (Fig. \ref{fig:illustration}).
The first two are binary classification problems with the
underlying class distributions shaped as
concentric circles (Fig. \ref{sfig:illustration_circles}),
and half moons (Fig. \ref{sfig:illustration_moons}),
respectively.
The third is a multi-class classification problem 
with Gaussian mixture class distributions 
(Fig. \ref{sfig:illustration_blobs}).
All datasets consist of $1500$ samples.
Since the objective is to give a geometric illustration 
of how the algorithm works in the two-dimensional plane, 
the Euclidean distance is used. 
The algorithm starts at high temperature 
with a single codevector for each class. 
As the temperature coefficient gradually decreases
(Fig. \ref{fig:illustration}, from left to right), 
the number of codevectors progressively increases.
The accuracy of the algorithm typically increases as well.
As the temperature goes to zero, the complexity of the model, 
i.e. the number of codevectors,
rapidly increases (Fig. \ref{fig:illustration}, rightmost pictures).
This may, or may not, translate to a corresponding performance boost. 
A single parameter --the temperature $T$-- 
offers online control on this complexity-accuracy trade-off.
Finally, Fig. \ref{sfig:illustration_initial} showcases the
robustness of the proposed algorithm with respect to 
the initial configuration.
Here the codevectors are poorly initialized 
outside the support of the data, which is not assumed known a priori
(e.g. online observations of unknown domain).
In this example the LVQ algorithm has been shown to fail \cite{aLaVigna_LVQconvergence_1990}.
In contrast, the entropy term $H$ in the optimization objective
of Alg. \ref{alg:ODA}, allows for the 
online adaptation to the domain of the dataset 
and helps to prevent poor local minima.


\subsection{Real Datasets}

\begin{figure}[t]
\centering
\begin{subfigure}[b]{0.24\textwidth}
\centering
\includegraphics[trim=0 0 0 0,clip,width=\textwidth]{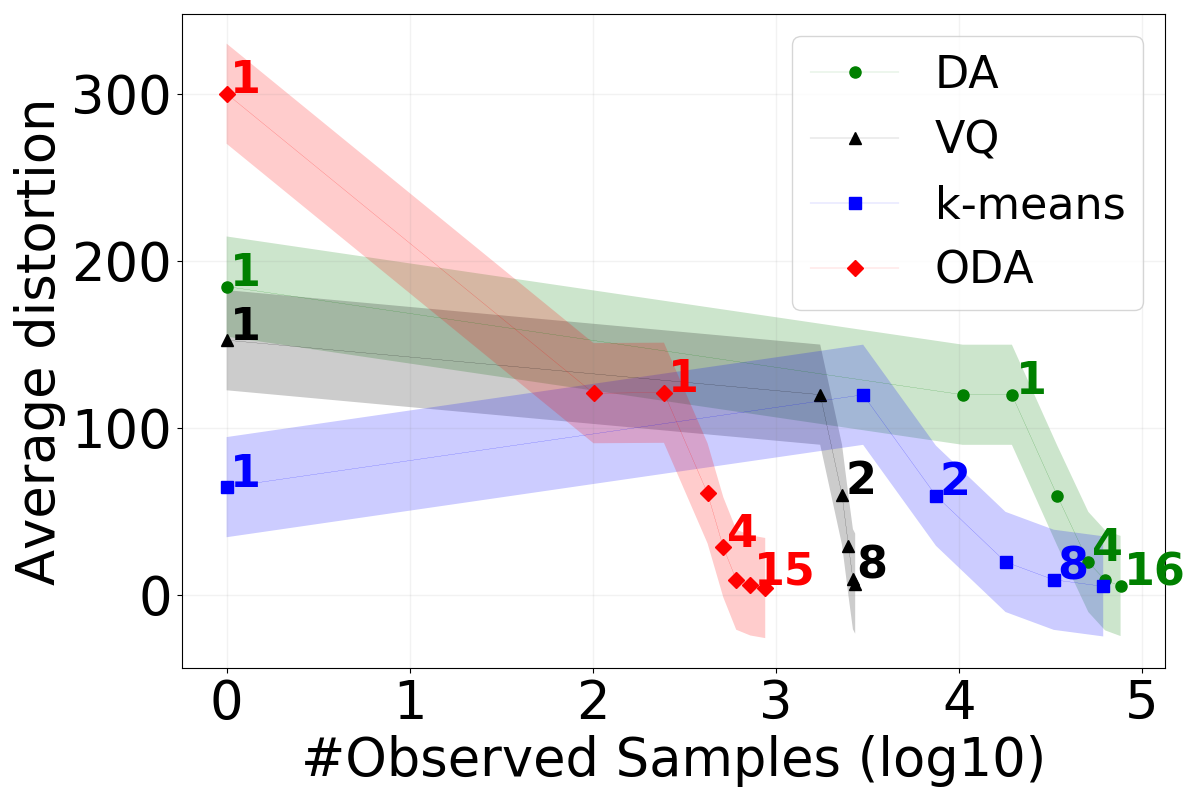}
\caption{Gaussians.}
\label{sfig:clustering_blobs}
\end{subfigure}
\hfill
\begin{subfigure}[b]{0.24\textwidth}
\centering
\includegraphics[trim=0 0 0 0,clip,width=\textwidth]{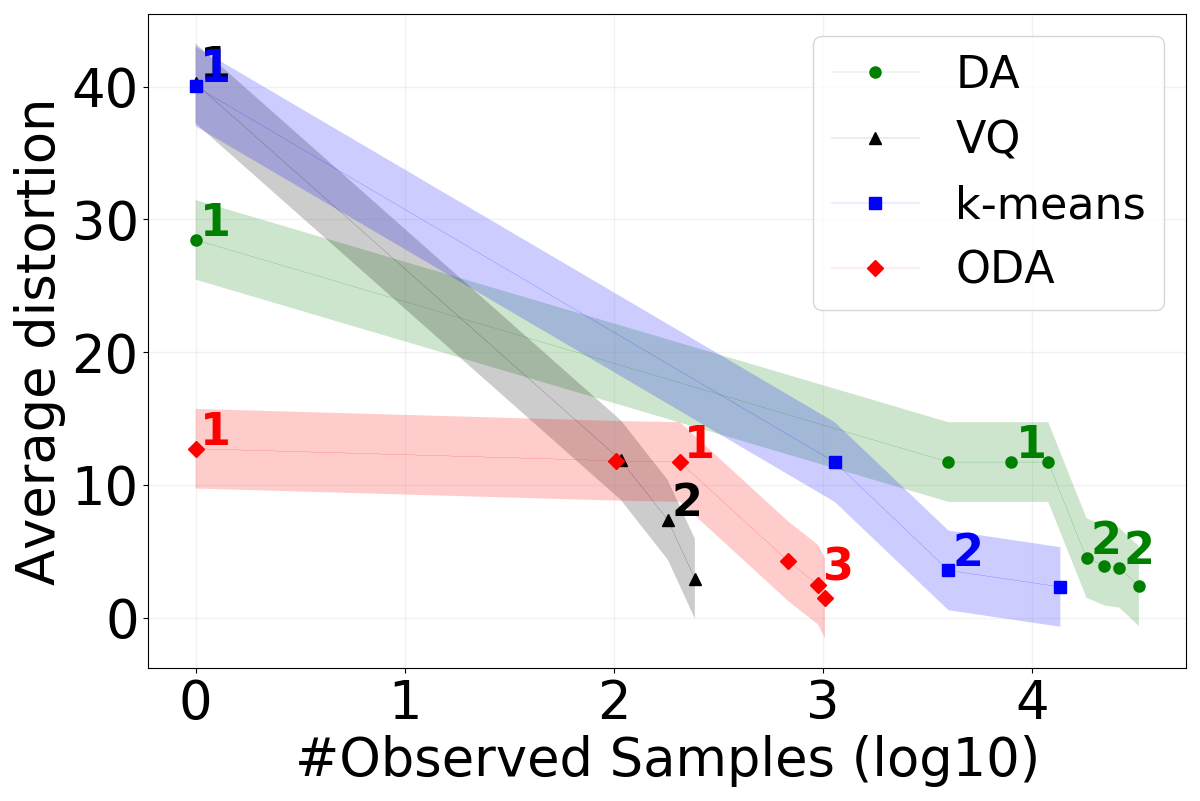}
\caption{WBCD.}
\label{sfig:clustering_wbcd}
\end{subfigure}
\hfill
\begin{subfigure}[b]{0.24\textwidth}
\centering
\includegraphics[trim=0 0 0 0,clip,width=\textwidth]{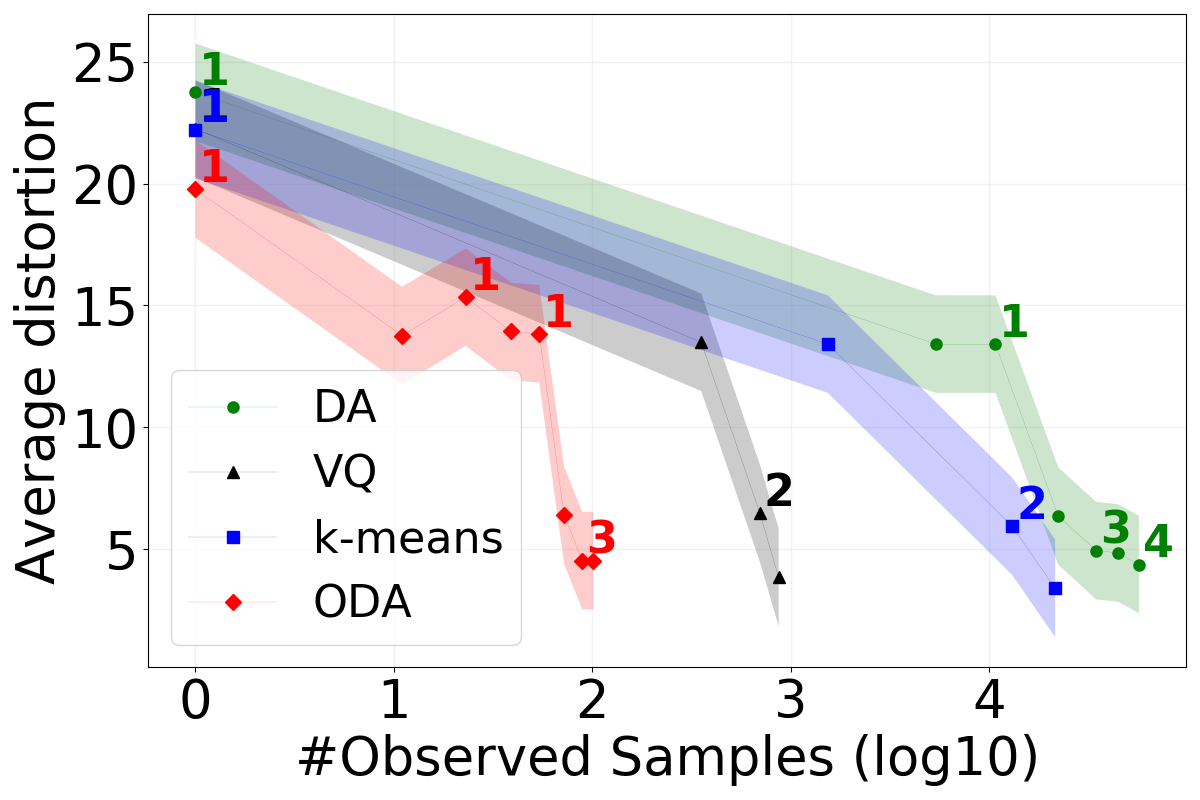}
\caption{PIMA.}
\label{sfig:clustering_pima}
\end{subfigure}
\hfill
\begin{subfigure}[b]{0.24\textwidth}
\centering
\includegraphics[trim=0 0 0 0,clip,width=\textwidth]{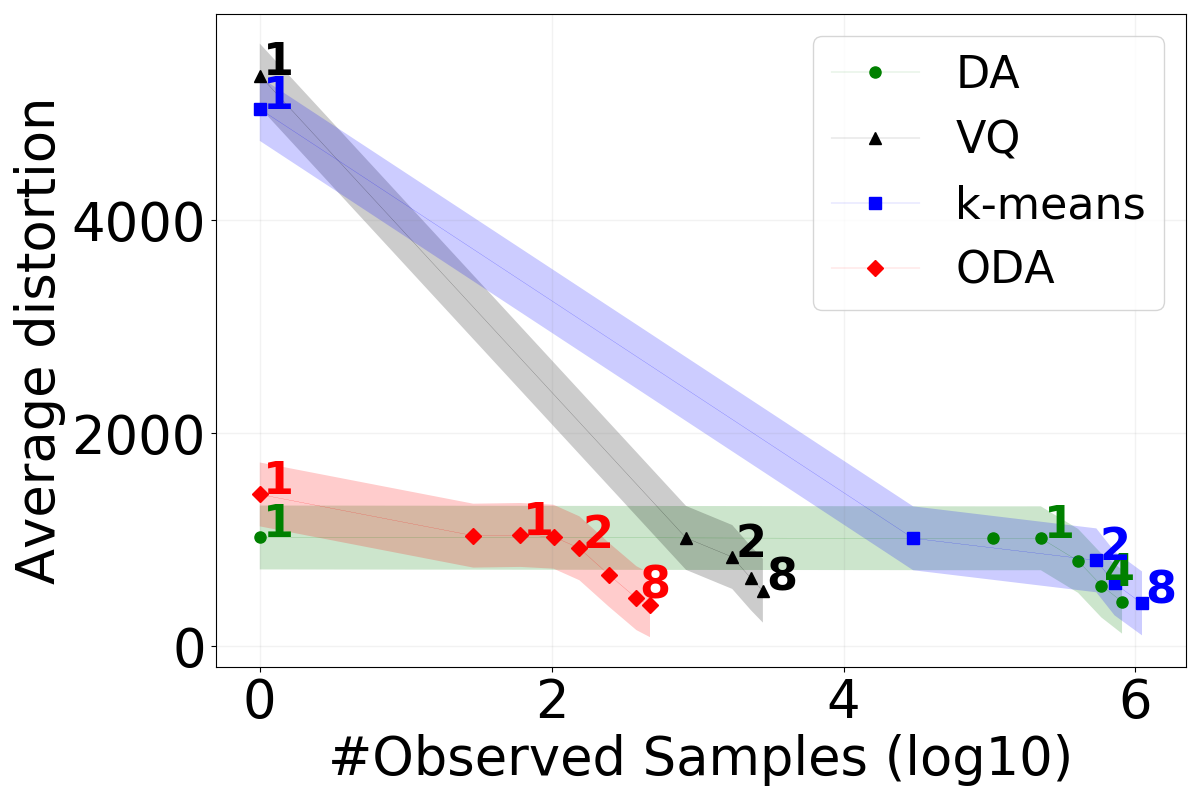}
\caption{Adult.}
\label{sfig:clustering_adult}
\end{subfigure}
\caption{Algorithm comparison for clustering.}
\label{fig:clustering}
\end{figure}

\textit{Clustering.}
For clustering, we consider the following datasets: 
(a) the dataset of Fig. \ref{sfig:illustration_blobs} (Gaussians),
(b) the WBCD dataset \cite{dua2019uci},
(c) the PIMA dataset \cite{smith1988using}, and
(d) the Adult dataset%
\footnote{$15000$ samples randomly selected. 
Non-numerical features removed.}%
\cite{dua2019uci}.
In Fig. \ref{fig:clustering}, 
we compare Alg. \ref{alg:ODA}
with the online sVQ algorithm (Def. \ref{def:sVQ}),
and two offline algorithms, namely
$k$-means \cite{bottou1995convergence},
and the original 
deterministic annealing (DA) algorithm \cite{rose1998deterministic}.
The algorithms are compared in terms of the minimum average distortion 
achieved, as a function of the number of samples they observed,
and 
the number of clusters they used (floating numbers inside the figures).
The Euclidean distance is used for fair comparison.
Since there is no criterion to decide the number of clusters $K$
for $k$-means and sVQ, we run them sequentially for the $K$ values
estimated by DA, and add up the computational time.
All algorithms are able to achieve 
comparable average distortion values given good initial conditions 
and appropriate size $K$. 
Therefore, the progressive estimation of $K$, as well as the robustness 
with respect to the initial 
conditions, are key features of both annealing algorithms, i.e., 
DA and ODA (Alg. \ref{alg:ODA}). 
Compared to the offline algorithms, i.e., $k$-means and DA, 
ODA and sVQ achieve practical convergence
with significantly smaller number of 
observations, which corresponds to 
reduced computational time, as argued above.
Notice the substantial difference in running time between the original DA algorithm 
and the proposed ODA algorithm in Fig. \ref{fig:rtime}. 
Compared to the online sVQ (and LVQ), 
the probabilistic approach of ODA introduces additional computational cost:  
all neurons are now updated in every iteration, 
instead of only the winner neuron.
However, the updates can still be computed fast when using
Bregman divergences (Theorem \ref{thm:bregman_in_DA}), 
and the aforementioned benefits of the annealing nature of ODA, 
outweigh this additional cost in many real-life problems.


%
\begin{figure}[t]
\centering
\begin{subfigure}[b]{0.24\textwidth}
\centering
\includegraphics[trim=0 0 0 0,clip,width=\textwidth]{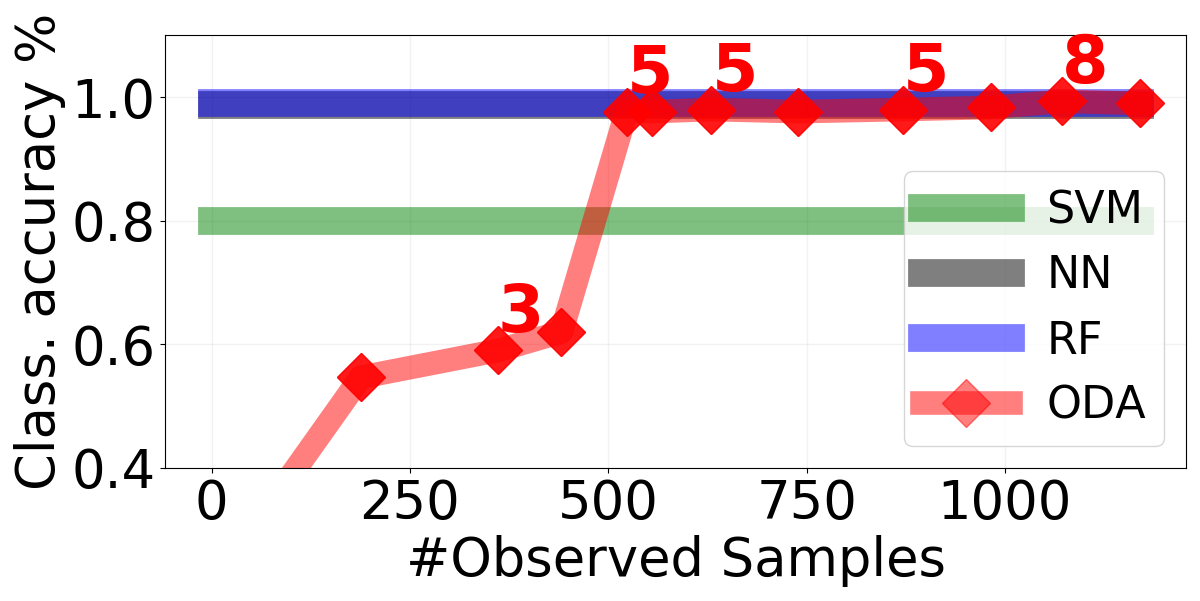}
\caption{Gaussians.}
\label{sfig:classification_blobs}
\end{subfigure}
\hfill
\begin{subfigure}[b]{0.24\textwidth}
\centering
\includegraphics[trim=0 0 0 0,clip,width=\textwidth]{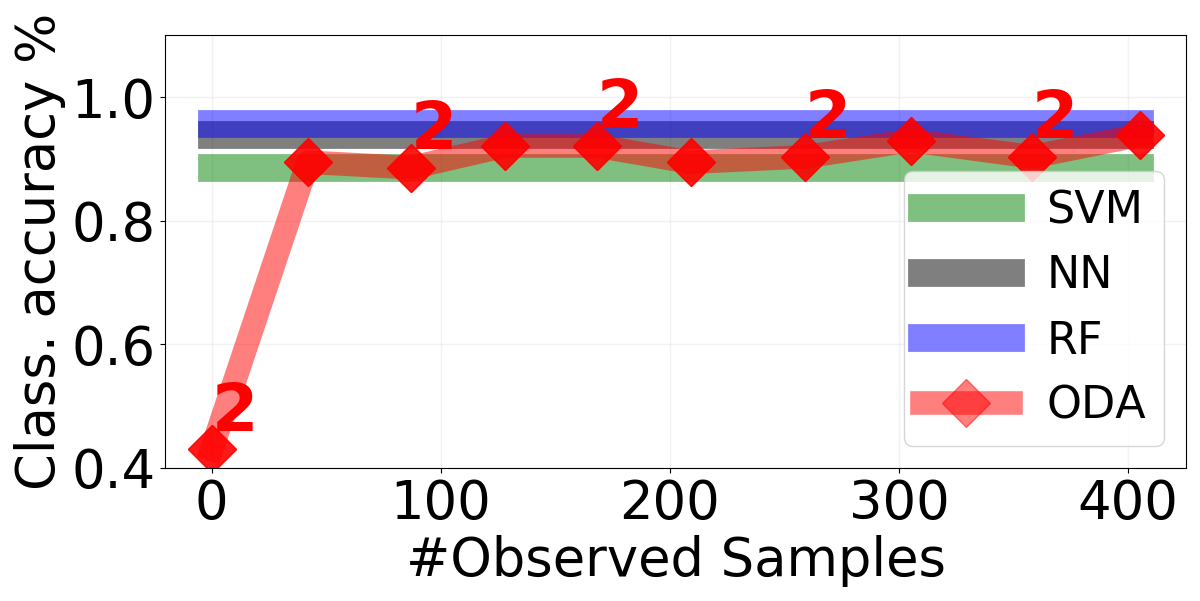}
\caption{WBCD.}
\label{sfig:classification_wbcd}
\end{subfigure}
\hfill
\begin{subfigure}[b]{0.24\textwidth}
\centering
\includegraphics[trim=0 0 0 0,clip,width=\textwidth]{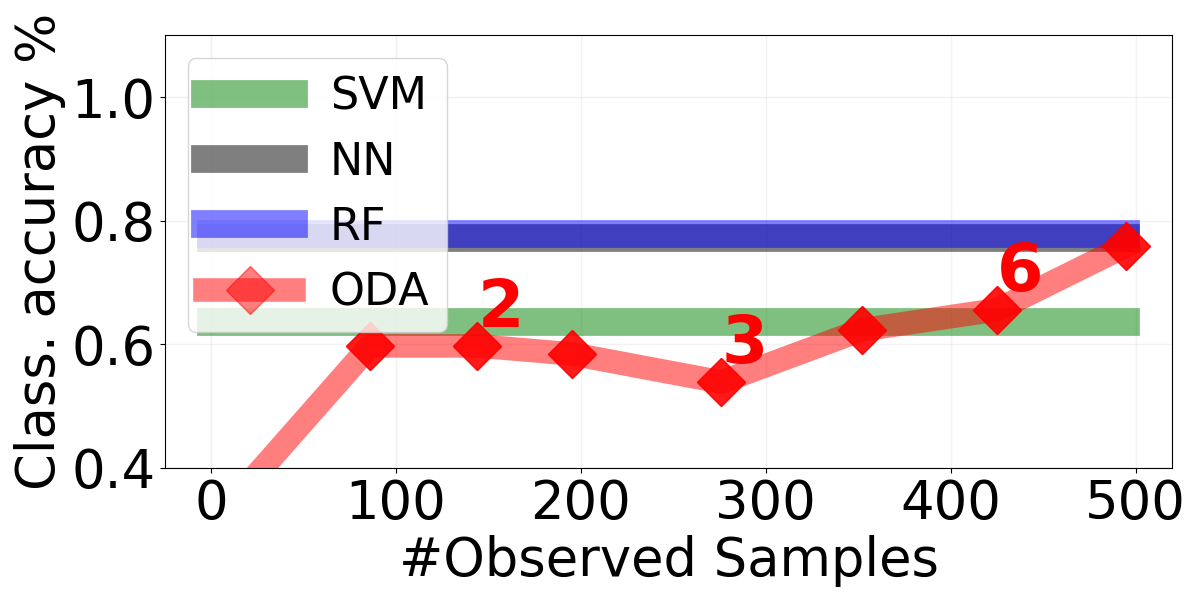}
\caption{PIMA.}
\label{sfig:classification_pima}
\end{subfigure}
\hfill
\begin{subfigure}[b]{0.24\textwidth}
\centering
\includegraphics[trim=0 0 0 0,clip,width=\textwidth]{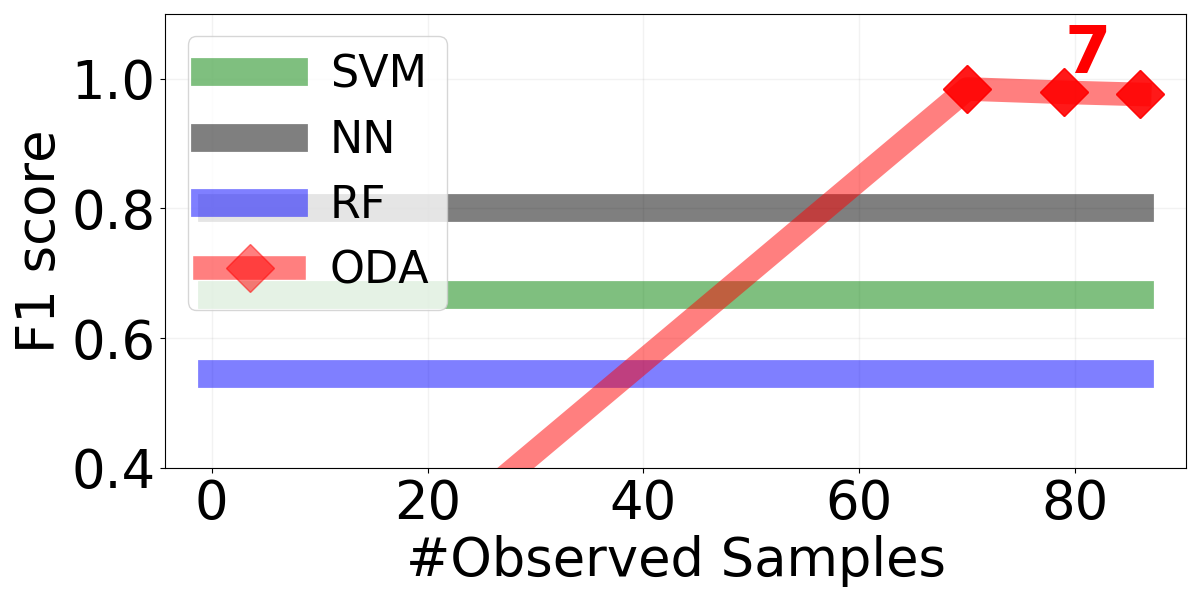}
\caption{Credit Card (F1 score).}
\label{sfig:classification_creditcard}
\end{subfigure}
\caption{Algorithm comparison for classification.}
\label{fig:classification}
\end{figure}

\begin{figure}[h]
\centering
\includegraphics[trim=0 0 0 0,clip,width=0.21\textwidth]{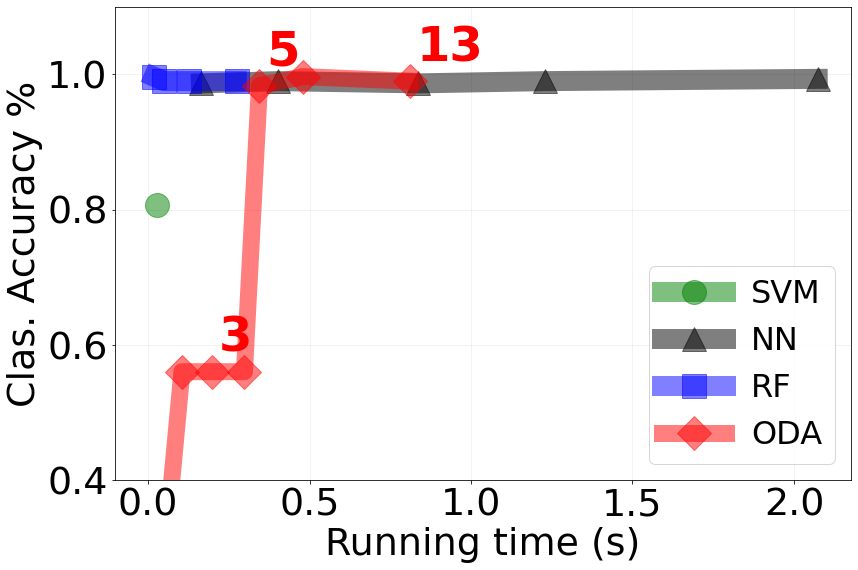}
\includegraphics[trim=0 0 0 0,clip,width=0.23\textwidth]{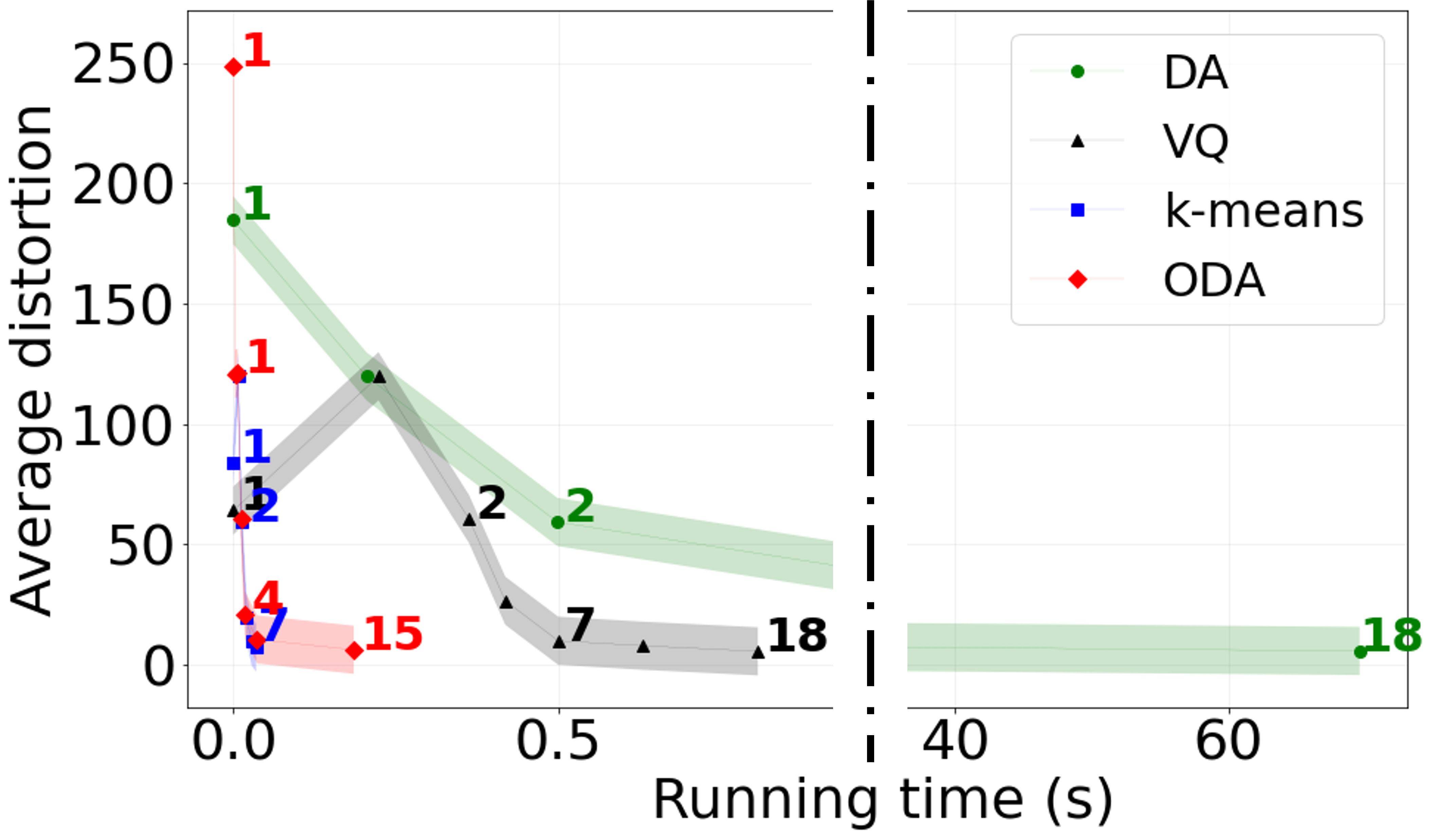}
\caption{Running time of the algorithms in Fig. \ref{sfig:clustering_blobs}, Fig. \ref{sfig:classification_blobs}.}
\label{fig:rtime}
\end{figure}

\begin{table}
\begin{center}
\begin{small}
\begin{sc}
\begin{tabular}{lcccc}
\toprule
Data set & ODA & SVM & NN & RF \\ 
\midrule
Gaussian	
	& 98.9{\tiny $\pm$ 0.0} & 79.5{\tiny $\pm$ 0.0} & 98.6{\tiny $\pm$ 0.0} & 98.7{\tiny $\pm$ 0.0} \\
WBCD    
	& 90.7{\tiny $\pm$ 0.0} & 85.6{\tiny $\pm$ 0.0} & 92.7{\tiny $\pm$ 0.0} & 94.6{\tiny $\pm$ 0.0} \\
Credit {\small (F1)}
	& 95.6{\tiny $\pm$ 0.0} & 69.1{\tiny $\pm$ 0.2} & 58.9{\tiny $\pm$ 0.1} & 62.8{\tiny $\pm$ 0.1} \\
PIMA	
	& 70.5{\tiny $\pm$ 0.0} & 62.9{\tiny $\pm$ 0.0} & 76.3{\tiny $\pm$ 0.0} & 74.4{\tiny $\pm$ 0.0} \\
\bottomrule
\end{tabular}
\end{sc}
\end{small}
\end{center}
\vskip -0.1in
\caption{Classification accuracies in $5$-fold cross-validation.}
\label{tbl:classification}
\end{table}

\textit{Classification.}
For classification, we consider the Gaussian 
(Fig. \ref{sfig:illustration_blobs}), WBCD, PIMA, and Credit Card%
\footnote{$15000$ samples randomly selected.}%
\cite{carcillo2019combining} datasets.
We compare Alg. \ref{alg:ODA}
against an SVM model with a linear kernel \cite{hearst1998support},
a feed-forward fully-connected neural network with a single hidden layer 
of $n_{NN}$ neurons (NN), and the Random Forests (RF) algorithm with 
$t_{RF}$ estimators \cite{breiman2001random}.
These algorithms have been selected to represent
today's standards in simple classification tasks,
i.e., when no sophisticated feature extraction is required.
The SVM classifier represents the class of linear classification models, 
the neural network represents the class of non-linear approximation models, 
and the random forests algorithm represents the class of partition-based
methods with bootstrap aggregating. 
Table \ref{tbl:classification} shows 
the results of a $5$-fold cross validation ($80/20\%$), and
Fig. \ref{fig:classification} illustrates the performance of 
the algorithms during a random test.
The evolution of the complexity of the ODA model is depicted as 
a function of the observed samples and the classification accuracy achieved.
%
%
We use the generalized I divergence (Example \ref{ex:Idiv}) in the WBCD dataset 
and the Euclidean distance in the rest.
ODA (Alg. \ref{alg:ODA}) outperforms the linear SVM classifier,
and can achieve comparable performance with the NN and the
RF algorithms, which are today's standards in classification tasks
where no feature extraction is required.
In the greatly unbalanced Credit Card dataset, 
all algorithms achieved accuracy close to $100\%$,
but the $F1$ scores 
dropped significantly (Fig. \ref{sfig:classification_creditcard}). 
Notably, this was not the case with the ODA algorithm.
This may be due to the generative nature of the algorithm, 
%
and might also be an instance of the robustness expected by 
vector quantization algorithms \cite{saralajew2019robustness}.
Justifying and quantifying this robustness is beyond the scope of this paper.

\textit{Parameters.}
The parameters $n_{NN}\in[10,100]$ and $t_{RF}\in[10,100]$ 
were selected through extensive grid search.
In contrast, the parameters of the ODA algorithm 
for all the experiments were set as follows:
$T_{max}=100\Delta_S d$, $T_{min}=0.001\Delta_S d$,
$K_{max}=100$, $\gamma=0.8$, $\epsilon_c=0.0001\Delta_S d$, 
$\epsilon_n=0.001\Delta_S d$, $\epsilon_r=10^{-7}$, $\delta=0.01\Delta_S d$,
and $\alpha_n= \nicefrac{1}{1+0.9n}$, where
$d$ is the number of dimensions of the input $X\in S\subseteq \mathbb{R}^d$, and 
$\Delta_S$ represents the length of the largest edge of the smallest $d$-orthotope 
that contains $S$.
We stress that no parameter tuning has taken place for the proposed algorithm.

\textit{Limitations.}
Finally, we note that both NN and RF 
outperform Alg. \ref{alg:ODA} in some datasets (Table \ref{tbl:classification}).   
A fine-tuning mechanism, as discussed in Section \ref{sSec:Algorithm},
could alleviate these differences, 
and is currently not used in our experiments.
Regarding the running time of the ODA algorithm, Fig. \ref{fig:rtime}
shows the execution time of the learning algorithms used in  
Fig. \ref{sfig:clustering_blobs} and Fig. \ref{sfig:classification_blobs}.
All experiments were implemented in a personal computer. 
We note that, in contrast to the commercial, and, therefore, optimized versions 
of the $k$-means, SVM, NN, and RF algorithms, 
the algorithmic implementation of the proposed algorithm is not yet optimized, 
and substantial speed-up is expected through appropriate software development.

%

\section{Conclusion}
\label{Sec:Conclusion}

It is understood that the trade-off between 
model complexity and performance in machine learning algorithms
is closely related to 
over-fitting, generalization, and robustness to 
input perturbations and adversarial attacks.
We investigate the properties of learning with progressively growing models,
and propose an online annealing optimization approach
as a learning algorithm
that progressively adjusts its complexity with respect to new observations, 
offering online control over the performance-complexity trade-off.
The proposed algorithm can be viewed as a neural network 
with inherent regularization mechanisms,
the learning rule of which is formulated as an online 
gradient-free stochastic approximation algorithm.
As a prototype-based learning algorithm, 
it offers a progressively growing knowledge base 
that can be interpreted as a memory unit that 
parallels similar concepts form cognitive psychology and neuroscience.
The annealing nature of the algorithm prevents poor local minima, 
offers robustness to initial conditions,
and provides a means 
to progressively increase the complexity of the learning model 
as needed.
%
To our knowledge, this is the first time such a progressive approach 
has been proposed for machine learning applications. 
We believe that these results can lead to new developments
in learning with progressively growing models, especially in 
communication, control, and reinforcement learning applications.

\bibliographystyle{IEEEtran} %
\bibliography{bib_learning.bib}


\vspace{-2.8em}

\begin{IEEEbiography}[{\includegraphics[width=1in,height=1.25in,clip,keepaspectratio]
{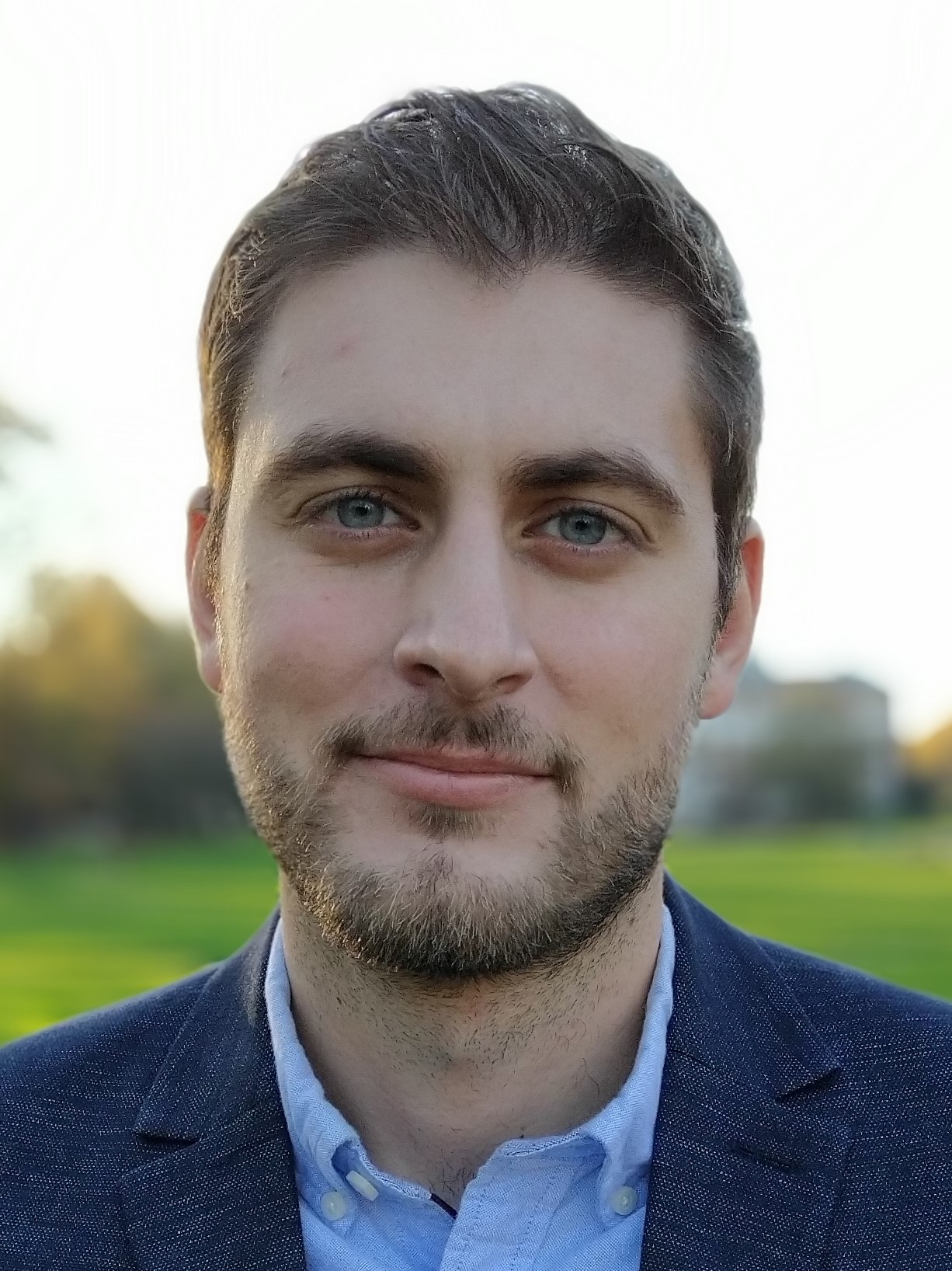}}]{Christos N. Mavridis} (M'20) 
received the Diploma degree in electrical and computer engineering from the National Technical University of Athens, Greece, in 2017,
and the M.S. and  Ph.D. degrees in electrical and computer engineering at the University of Maryland, College Park, MD, USA, in 2021. 
His research interests include learning theory, stochastic optimization, systems and control theory, multi-agent systems, and robotics. 

He has worked as a researcher at the Department of Electrical and Computer Engineering at the University of Maryland, College Park, and as a research intern for the Math and Algorithms Research Group at Nokia Bell Labs, NJ, USA, and the System Sciences Lab at Xerox Palo Alto Research Center (PARC), CA, USA. 

Dr. Mavridis is an IEEE member, and a member of the Institute for Systems Research (ISR) and the Autonomy, Robotics and Cognition (ARC) Lab. He received the Ann G. Wylie Dissertation Fellowship in 2021, and the A. James Clark School of Engineering Distinguished Graduate Fellowship, Outstanding Graduate Research Assistant Award, and Future Faculty Fellowship, in 2017, 2020, and 2021, respectively. He has been a finalist in the Qualcomm Innovation Fellowship US, San Diego, CA, 2018, and he has received the Best Student Paper Award (1st place) in the IEEE International Conference on Intelligent Transportation Systems (ITSC), 2021.
\end{IEEEbiography}

\vspace{-2.8em}

\begin{IEEEbiography}[{\includegraphics[width=1in,height=1.25in,clip,keepaspectratio]
{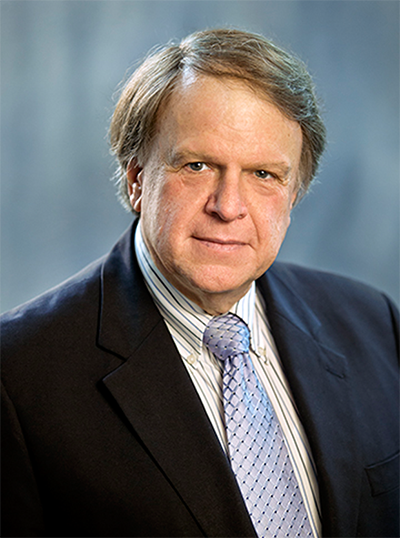}}]{John S. Baras} (LF'13) 
received the Diploma degree in electrical and mechanical engineering from the National Technical University of Athens, Greece, in 1970, and the M.S. and Ph.D. degrees in applied mathematics from Harvard University, Cambridge, MA, USA, in 1971 and 1973, respectively.

He is a Distinguished University Professor and holds the Lockheed Martin Chair in Systems Engineering, with the Department of Electrical and Computer Engineering and the Institute for Systems Research (ISR), at the University of Maryland College Park. From 1985 to 1991, he was the Founding Director of the ISR. Since 1992, he has been the Director of the Maryland Center for Hybrid Networks (HYNET), which he co-founded. His research interests include systems and control, optimization, communication networks, applied mathematics, machine learning, artificial intelligence, signal processing, robotics, computing systems, security, trust, systems biology, healthcare systems, model-based systems engineering.

Dr. Baras is a Fellow of IEEE (Life), SIAM, AAAS, NAI, IFAC, AMS, AIAA, Member of the National Academy of Inventors and a Foreign Member of the Royal Swedish Academy of Engineering Sciences. Major honors include the 1980 George Axelby Award from the IEEE Control Systems Society, the 2006 Leonard Abraham Prize from the IEEE Communications Society, the 2017 IEEE Simon Ramo Medal, the 2017 AACC Richard E. Bellman Control Heritage Award, the 2018 AIAA Aerospace Communications Award. In 2016 he was inducted in the A. J. Clark School of Engineering Innovation Hall of Fame. In 2018 he was awarded a Doctorate Honoris Causa by his alma mater the National Technical University of Athens, Greece.   
\end{IEEEbiography}

\end{document}